\documentclass[11pt]{article}
\usepackage{enumerate}
\usepackage[OT1]{fontenc}
\usepackage{amsmath,amssymb}
\usepackage{natbib}
\usepackage[usenames]{color}

\usepackage{smile}
\usepackage{multirow}
\usepackage[breaklinks,colorlinks,
            linkcolor=red,
            anchorcolor=blue,
            citecolor=blue
            ]{hyperref}

\def\given{\,|\,}

\def\tr{\mathop{\text{tr}}\kern.2ex}

\long\def\comment#1{}

\def\tr{\mathop{\text{Tr}}}

\def\cS{{\mathcal{S}}}

\def\cX{{\mathcal{X}}}

\def\cP{{\mathcal{P}}}

\def\cT{{\mathcal{T}}}

\def\tr{{\text{Tr}}}

\newcommand{\bel}{\begin{eqnarray}\label}
\newcommand{\eel}{\end{eqnarray}}
\newcommand{\bes}{\begin{eqnarray*}}
\newcommand{\ees}{\end{eqnarray*}}

\usepackage{mathrsfs}

\usepackage{fullpage}

\newcommand{\rd}{{\mathrm{d}}}

\newcommand{\iid}{{\text{i.i.d.}}}
\newcommand{\unif}{{\text{Unif}}}
\newcommand{\kl}{{\mathrm{KL}}}

\newcommand{\rE}{{\text{{\tiny E}}}}

\newcommand{\TD}{{\mathrm{TD} } }
\newcommand{\init}{{\text{init}}}
\newcommand{\proj}{\text{Proj}}
\newcommand{\err}{{\epsilon}}
\newcommand{\phic}{\iota}
\newcommand{\DD}{\mathbb{D}}
\def\td{{\mathrm{TD}}}

\usepackage{mathtools}
\DeclarePairedDelimiterX{\inp}[2]{\langle}{\rangle}{#1, #2}
\DeclarePairedDelimiterX{\norm}[1]{\|}{\|}{#1}

\usepackage{hyperref}
\usepackage[english]{babel}
\usepackage{breakcites}
\usepackage{microtype}
\def\##1\#{\begin{align}#1\end{align}}
\def\$#1\${\begin{align*}#1\end{align*}}
\usepackage{enumitem}

\usepackage{setspace}

\title{Generative Adversarial Imitation Learning with Neural Networks: Global Optimality and Convergence Rate}
\author
{
	\normalsize Yufeng Zhang\thanks{Northwestern University; \texttt{yufengzhang2023@u.northwestern.edu}}
	\qquad
	\normalsize Qi Cai~\thanks{Northwestern University; \texttt{qicai2022@u.northwestern.edu}}
	\qquad
	\normalsize Zhuoran Yang\thanks{Princeton University; \texttt{zy6@princeton.edu}}
	\qquad
	\normalsize Zhaoran Wang\thanks{Northwestern University; \texttt{zhaoranwang@gmail.com}}
}
\date{}

\begin{document}	
\maketitle


\begin{abstract}
	Generative adversarial imitation learning (GAIL) demonstrates tremendous success in practice, especially when combined with neural networks. Different from reinforcement learning, GAIL learns both policy and reward function from expert (human) demonstration. Despite its empirical success, it remains unclear whether GAIL with neural networks converges to the globally optimal solution. The major difficulty comes from the nonconvex-nonconcave minimax optimization structure. To bridge the gap between practice and theory, we analyze a gradient-based algorithm with alternating updates and establish its sublinear convergence to the globally optimal solution. To the best of our knowledge, our analysis establishes the global optimality and convergence rate of GAIL with neural networks for the first time.
\end{abstract}

\section{Introduction}
The goal of imitation learning (IL) is to learn to perform a task based on expert demonstration \citep{ho2016generative}. In contrast to reinforcement learning (RL), the agent only has access to the expert trajectories but not the rewards. The most straightforward approach of IL is behavioral cloning (BC) \citep{pomerleau1991efficient}. BC treats IL as the supervised learning problem of predicting the actions based on the states. Despite its simplicity, BC suffers from the compounding errors caused by covariate shift \citep{ross2011reduction, ross2010efficient}. Another approach of IL is inverse reinforcement learning (IRL) \citep{russell1998learning, ng2000algorithms, levine2012continuous, finn2016guided}, which jointly learns the reward function and the corresponding optimal policy. IRL formulates IL as a bilevel optimization problem. Specifically, IRL solves an RL subproblem given a reward function at the inner level and searches for the reward function which makes the expert policy optimal at the outer level. However, IRL is computationally inefficient as it requires fully solving an RL subproblem at each iteration of the outer level. Moreover, the desired reward function may be nonunique. To address such issues of IRL, \cite{ho2016generative} propose generative adversarial imitation learning (GAIL), which searches for the optimal policy without fully solving an RL subproblem given a reward function at each iteration. GAIL solves IL via minimax optimization with alternating updates. In particular, GAIL alternates between (i) minimizing the discrepancy in expected cumulative reward between the expert policy and the learned policy and (ii) maximizing such a discrepancy over the reward function class. Such an alternating update scheme mirrors the training of generative adversarial networks (GANs) \citep{goodfellow2014generative, arjovsky2017wasserstein}, where the learned policy acts as the generator while the reward function acts as the discriminator.

Incorporated with neural networks, which parameterize the learned policy and the reward function, GAIL achieves significant empirical success in challenging applications, such as natural language processing \citep{yu2016seqgan}, autonomous driving \citep{kuefler2017imitating}, human behavior modeling \citep{merel2017learning}, and robotics \citep{tai2018socially}. 
Despite its empirical success, GAIL with neural networks remains less understood in theory. The major difficulty arises from the following aspects:
(i) GAIL involves minimax optimization, while the existing analysis of policy optimization with neural networks \citep{anthony2009neural,  liu2019neural, bh2019global, wang2019neural} only focuses on a minimization or maximization problem.
(ii) GAIL with neural networks is nonconvex-nonconcave, and therefore, the existing analysis of convex-concave optimization with alternating updates is not applicable \citep{nesterov2013introductory}. There is an emerging body of literature \citep{rafi2018noncon, zhang2019policy} that casts nonconvex-nonconcave optimization as bilevel optimization, where the inner level is solved to a high precision as in IRL. However, such analysis is not applicable to GAIL as it involves alternating updates.

In this paper, we bridge the gap between practice and theory by establishing the global optimality and convergence of GAIL with neural networks.  Specifically, we parameterize the learned policy and the reward function with two-layer neural networks and consider solving GAIL by alternatively updating the learned policy via a step of natural policy gradient \citep{kakade2002natural, peters2008natural} and the reward function via a step of gradient ascent.
In particular, we parameterize the state-action value function (also known as the Q-function) with a two-layer neural network and apply a variant of the temporal difference algorithm \citep{sutton2018reinforcement} to solve the policy evaluation subproblem in natural policy gradient.   
We prove that the learned policy $\bar\pi$ converges to the expert policy $\pi_\rE$ at a $1/\sqrt{T}$ rate in the $\cR$-distance \citep{chen2020computation}, which is defined as
$
\DD_{\cR}(\pi_\rE, \bar \pi) = \max_{r\in\cR} J(\pi_\rE; r) - J(\bar \pi; r). 
$
Here $J(\pi; r)$ is the expected cumulative reward of a policy $\pi$ given a reward function $r(s, a)$ and $\cR$ is the reward function class.
The core of our analysis is constructing a potential function that tracks the $\cR$-distance. Such a rate of convergence implies that the learned policy $\bar \pi$ (approximately)  outperforms the expert policy $\pi_\rE$ given any reward function $r\in \cR$ within a finite number of iterations $T$. In other words, the learned policy $\bar \pi$ is globally optimal.
To the best of our knowledge, our analysis establishes the global optimality and convergence of GAIL with neural networks for the first time. It is worth mentioning that  our analysis is straightforwardly applicable to linear and tabular settings, which, however, are not our focus.

\vskip4pt

\noindent{\bf Related works.} Our work extends an emerging body of literature on RL with neural networks \citep{xu2019improved, zhang2019global, bh2019global, liu2019neural, wang2019neural,  agarwal2019optimality} to IL. This line of research analyzes the global optimality and convergence of policy gradient for solving RL, which is a minimization or maximization problem. In contrast, we analyze GAIL, which is a minimax optimization problem.

Our work is also related to the analysis of apprenticeship learning \citep{syed2008apprenticeship} and GAIL \citep{cai2019global, chen2020computation}. \cite{syed2008apprenticeship} analyze the convergence and generalization of apprenticeship learning. They assume the state space to be finite, and thus, do not require function approximation for the policy and the reward function. In contrast, we assume the state space to be infinite and employ function approximation based on neural networks.
\cite{cai2019global} study the global optimality and convergence of GAIL in the setting of linear-quadratic regulators. In contrast, our analysis handles general MDPs without restrictive assumptions on the transition kernel and the reward function.
\cite{chen2020computation} study the convergence and generalization of GAIL in the setting of general MDPs. However, they only establish the convergence to a stationary point. In contrast, we establish the global optimality of GAIL.

%
%
\vskip4pt

\noindent{\bf Notations.} Let $[n] = \{1, \ldots, n\}$ for $n\in\NN_{+}$ and $[m:n] = \{m, m+1, \ldots, n\}$ for $m < n$. Also, let $N(\mu, \Sigma)$ be the Gaussian distribution with mean $\mu$ and covariance $\Sigma$. We denote by $\cP(\cX)$ the set of all probability measures over the space $\cX$. For a function $f:\cX\rightarrow \RR$, a constant $p \ge 1$, and a probability measure $\mu\in \cP(\cX)$, we denote by $\norm{f}_{p, \mu} = (\int_\cX |f(x)|^p \rd \mu(x))^{1/p}$ the $L_p(\mu)$ norm of the function $f$. For any two functions $f,g:\cX\rightarrow \RR$, we denote by $\inp{f}{g}_\cX = \int_\cX f(x)\cdot g(x) \rd x$ the inner product on the space $\cX$. 


\section{Background} \label{sec:bg}
\vskip4pt
In this section, we introduce reinforcement learning (RL) and generative adversarial imitation learning (GAIL). 

\subsection{Reinforcement Learning}
\label{sec:mdp}
We consider a Markov decision process (MDP) $(\cS, \cA, r, P, \rho, \gamma)$. Here $\cS \subseteq \RR^{d_1}$ is the state space, $\cA \subseteq \RR^{d_2} $ is the action space, which is assumed to be finite throughout this paper, $r:\cS \times \cA \rightarrow \RR$ is the reward function, $P:\cS\times\cA\rightarrow \cP(\cS)$ is the transition kernel, $\rho\in\cP(\cS)$ is the initial state distribution, and $\gamma \in (0,1)$ is the discount factor. Without loss of generality, we assume that $\cS\times \cA$ is compact and that $\norm{(s, a)}_2 \le 1$ for any $(s, a)\in\cS\times \cA \subseteq \RR^d$, where $d = d_1+d_2$.
An agent following a policy $\pi : \cS \rightarrow \cP(\cA) $ interacts with the environment in the following manner. At the state $s_t \in \cS$, the agent takes the action $a_t \in \cA$ with probability $\pi(a_t\given s_t)$ and receives the reward $r_t = r(s_t, a_t)$. The environment then transits into the next state $s_{t+1}$ with probability $P(s_{t+1} \given s_t, a_t)$. Given a policy $\pi$ and a reward function $r(s, a)$, we define the state-action value function $Q^\pi_r: \cS\times \cA \rightarrow \RR$ as follows,
\begin{align}
	Q^\pi_r(s, a) = \EE_\pi \biggl[(1-\gamma)\cdot  \sum_{t=0}^\infty \gamma^t \cdot r(s_t, a_t) \,\bigg| \, s_0 = s, a_0 = a \biggr]. \label{eq:def_q}
\end{align}
Here the expectation $\EE_\pi$ is taken with respect to $a_t\sim \pi(\cdot \given s_t)$ and $s_{t+1} \sim P(\cdot \given s_t, a_t)$.
Correspondingly, we define the state value function $V^\pi_r:\cS\rightarrow \RR$ and the advantage function $A^\pi_r:\cS\times \cA \rightarrow \RR$ as follows,
\begin{align*}
	V^\pi_r(s) = \EE_{a \sim \pi(\cdot \given s)} \bigl[ Q_r^\pi (s, a)  \bigr],  \quad
	A^\pi_r (s, a)  = Q_r^\pi(s, a) - V^\pi_r (s). 
\end{align*}
The goal of RL is to maximize the following expected cumulative reward,
\begin{align}
\label{eq:def_rl}
J(\pi; r) = \EE_{s\sim \rho}\bigl[V_{r}^{\pi} (s)\bigr].
\end{align}
The policy $\pi$ induces a state visitation measure $d_\pi\in\cP(\cS)$ and a state-action visitation measure $\nu_\pi\in \cP(\cS\times \cA)$, which take the forms of
\begin{align}
\label{eq:def_visitation}
d_\pi(s) = (1- \gamma)\cdot \sum_{t=0}^\infty \gamma^t \cdot \PP\bigl(s_t = s \,\big|\, s_0 \sim \rho, a_t \sim \pi(\cdot \,|\, s_t)\bigr), \quad \nu_\pi(s, a) = d_\pi(s) \cdot \pi(a\given s).
\end{align}
It then holds that $J(\pi; r) = \EE_{(s, a) \sim \nu_\pi}[r(s, a)]$.
Meanwhile, we assume that the policy $\pi$ induces a state stationary distribution $\varrho_\pi$ over $\cS$, which satisfies that
\begin{align*}
\varrho_\pi(s) = \PP\bigl(s_{t+1} = s \,\big|\, s_t \sim \rho_\pi, a_t \sim \pi(\cdot \given s_t)\bigr).
\end{align*}
We denote by $\rho_\pi(s, a) = \varrho(s) \cdot \pi(a\given s)$ the state-action stationary distribution over $\cS\times \cA$.

\subsection{Generative Adversarial Imitation Learning} \label{sec:il}
The goal of imitation learning (IL) is to learn a policy that outperforms the expert policy $\pi_\rE$ based on the trajectory $\{(s_i^\rE,a_i^\rE)\}_{i\in[T_\rE]}$ of $\pi_\rE$.  We denote by $\nu_\rE = \nu_{\pi_\rE}$ and $d_\rE = d_{\pi_\rE}$ the state-action and state visitation measures induced by the expert policy, respectively, and assume that the expert trajectory $\{(s_i,a_i)\}_{i\in[T_\rE]}$ is drawn from $\nu_{\rE}$. 
To this end, we parameterize the policy and the reward function by $\pi_\theta$ for $\theta\in\cX_\Pi$ and $r_\beta(s, a)$ for $\beta \in \cX_R$, respectively, and solve the following minimax optimization problem known as GAIL \citep{ho2016generative},
\begin{align}
\label{eq:def_il}
\min_{\theta\in \cX_\Pi} \max_{\beta \in \cX_R}  L(\theta, \beta), \quad \text{where  } L(\theta, \beta) = J(\pi_\rE; r_\beta) - J(\pi_\theta; r_\beta) - \lambda \cdot \psi(\beta).
\end{align}
Here $J(\pi; r)$ is the expected cumulative reward defined in \eqref{eq:def_rl}, $\psi:\cX_R\rightarrow \RR_+$ is the regularizer, and $\lambda \ge 0$ is the regularization parameter. Given a reward function class $\cR$, we define the $\cR$-distance between two policies $\pi_1$ and $\pi_2$ as follows,
\begin{align}
\label{eq:def_dis}
\DD_{\cR}(\pi_1, \pi_2) = \max_{r\in\cR} J(\pi_1;r) - J(\pi_2; r) = \max_{r \in \cR} \EE_{\nu_{\pi_1}}\bigl[ r(s, a)\bigr] - \EE_{\nu_{\pi_2}} \bigl[ r(s, a) \bigr].
\end{align}
When $\cR$ is the class of $1$-Lipschitz functions, $\DD_{\cR}(\pi_1, \pi_2)$ is the Wasserstein-$1$ metric between the state-action visitation measures induced by $\pi_1$ and $\pi_2$. However, $\DD_{\cR}(\pi_1, \pi_2)$ is not a metric in general.
When $\DD_{\cR}(\pi_1, \pi_2) \le 0$, the policy $\pi_2$ outperforms the policy $\pi_1$ for any reward function $r \in \cR$. Such a notion of $\cR$-distance is used in \cite{chen2020computation}. We denote by $\cR_\beta = \{r_\beta(s, a) \given \beta \in \cX_R\}$ the reward function class parameterized with $\beta$. Hence, the optimization problem in \eqref{eq:def_il} minimizes the $\cR_\beta$-distance $\DD_{\cR_\beta}(\pi_\rE, \pi_\theta)$ (up to the  regularizer $\lambda\cdot \psi(\beta)$), which searches for a policy $\bar \pi$ that (approximately) outperforms the expert policy given any reward function $r_{\beta} \in \cR_\beta$.



\section{Algorithm} \label{sec:alg}
In this section, we introduce an algorithm with alternating updates for GAIL with neural networks, which employs natural policy gradient (NPG) to update the policy $\pi_\theta$ and gradient ascent to update the reward function $r_\beta(s, a)$.

\subsection{Parameterization with Neural Networks} \label{sec:nn}
We define a two-layer neural network with rectified linear units (ReLU) as follows,
\begin{align}
\label{eq:def_nn}
u_{W, b}(s, a) = \frac{1}{\sqrt{m}} \sum_{l=1}^{m} b_{l} \cdot \ind{\bigl\{(s, a)^\top [W]_l>0 \bigr\}}\cdot (s, a)^\top [W]_l 
= \sum_{l=1}^{m} \bigl[\phi_{W, b}(s, a)\bigr]_l^\top [W]_l.
\end{align}
Here  $m\in \NN_+$ is the width of the neural network, $b = (b_1, \ldots, b_m)^\top \in \RR^m$ and $W = ([W]_1^\top, \ldots , [W]_m^\top)^\top \in \RR^{md}$ are the parameters, and $\phi_{W, b}(s, a) = ([\phi_{W, b}(s, a)]^\top_1, \ldots, [\phi_{W, b}(s, a)]^\top_m)^\top \in \RR^{md}$ is called the feature vector in the sequel, where
\begin{align}
\label{eq:def_feature}
\bigl[\phi_{W, b}(s, a)\bigr]_l = m^{-1/2}\cdot b_l \cdot \ind\bigl\{ (s, a)^\top [W]_l > 0\bigr\}\cdot (s, a).
\end{align}
It then holds that $u_{W, b}(s, a) = W^\top \phi_{W, b}(s, a)$. Note that the feature vector $\phi_{W, b}(s, a)$ depends on the parameters $W$ and $b$. We consider the following random initialization,
\begin{align}
\label{eq:init}
b_l \overset{\iid}{\sim} \unif\bigl(\{-1, 1\}\bigr),\quad 
[W_0]_l \overset{\iid}{\sim} N(0, I_d / d),\quad \forall l\in [m]. 
\end{align}
Throughout the training process, we keep the parameter $b$ fixed while updating $W$. For notational simplicity, we write $u_{W, b}(s, a)$ as $u_{W}(s, a)$ and $\phi_{W, b}(s, a)$ as $\phi_{W}(s, a)$ in the sequel. We denote by $\EE_\init$ the expectation taken with respect to the random initialization in \eqref{eq:init}. For an absolute constant $B>0$, we define the parameter domain as
\begin{align}
\label{eq:def_s_b}
S_B = \bigl\{W\in\RR^{md}\,\big|\, \norm{W-W_0}_2 \le B \bigr\},
\end{align}
which is the ball centered at $W_0$ with the domain radius $B$. 

In the sequel, we consider the following energy-based policy,
\begin{align}
\label{eq:para_pi}
\pi_\theta(a\given s) = \frac{\exp\bigl(\tau \cdot u_\theta(s, a)\bigr)}{\sum_{a'\in \cA}\exp\bigl(\tau \cdot u_\theta(s, a')\bigr)},
\end{align}
where $\tau \ge 0$ is the inverse temperature parameter and $u_\theta(s, a)$ is the energy function parameterized by the neural network defined in \eqref{eq:def_nn} with $W=\theta$. In the sequel, we call $\theta$ the policy parameter. Meanwhile, we parameterize the reward function $r_\beta(s, a)$ as follows,
\begin{align}
\label{eq:para_r}
r_\beta(s, a) = (1-\gamma)^{-1}\cdot u_\beta(s, a),
\end{align}
where $u_\beta(s, a)$ is the neural network defined in \eqref{eq:def_nn} with $W= \beta$ and $\gamma$ is the discount factor. Here we use the scaling parameter $(1-\gamma)^{-1}$ to ensure that for any policy $\pi$ the state-action value function $Q^\pi_{r_\beta}(s, a)$ defined in \eqref{eq:def_q} is well approximated by the neural network defined in \eqref{eq:def_nn}.
In the sequel, we call $\beta$ the reward parameter and define the reward function class as 
\$\cR_{\beta} = \{r_{\beta}(s, a) \given \beta \in S_{B_{\beta}}\}, \$
where $S_{B_{\beta}}$ is the parameter domain of $\beta$ defined in \eqref{eq:def_s_b} with domain radius $B_{\beta}$. 
To facilitate algorithm design, we establish the following proposition, which calculates the explicit expressions of the gradients $\nabla L(\theta, \beta)$ and the Fisher information $\cI(\theta)$. Recall that the Fisher information is defined as
\begin{align}
\label{eq:fisher}
\cI(\theta) = \EE_{(s, a)\sim \nu_{\pi_\theta}}\bigl[\nabla_\theta \log \pi_\theta(s, a)\nabla_\theta \log \pi_\theta (s, a)^\top\bigr].
\end{align}
\begin{proposition}[Gradients and Fisher Information]
	\label{prop:npg_form}
	We call $\phic_\theta(s, a) = \tau^{-1}\cdot \nabla_\theta \log \pi_\theta(a\given s)$ the temperature-adjusted score function. It holds that
	\begin{align}
	\label{eq:def_phic}
	\phic_\theta(s, a) = \phi_{\theta}(s, a) - \EE_{a' \sim \pi_\theta(\cdot \given s)}\bigl[\phi_{\theta}(s, a')\bigr].
	\end{align}	
It then holds that
	\begin{align}
	\label{eq:fisher_form}
	\cI(\theta) &= \tau^2 \cdot \EE_{(s,a)\sim \nu_{\pi_\theta} }\bigl[\phic_\theta(s, a) \, \phic_\theta (s, a)^\top \bigr], \\
	\label{eq:pg_form}
	\nabla_\theta L(\theta, \beta) &=  -\tau \cdot \EE_{(s,a) \sim \nu_{\pi_\theta}} \bigl[Q^{\pi_\theta}_{r_\beta} (s, a) \cdot \phic_\theta(s, a)\bigr], \\
	\label{eq:rg_form}
	\nabla_\beta L(\theta, \beta) &= (1-\gamma)^{-1} \cdot \EE_{(s, a)\sim \nu_{\rE}}\bigl[\phi_\beta(s, a)\bigr]  - (1-\gamma)^{-1} \cdot \EE_{(s, a)\sim \nu_{\pi_\theta}}\bigl[\phi_\beta(s, a)\bigr]- \lambda \cdot \nabla_\beta \psi(\beta),
	\end{align}
	where $Q^{\pi_\theta}_{r_\beta}(s, a)$ is the state-action value function defined in \eqref{eq:def_q} with $\pi = \pi_\theta$ and $r = r_\beta$, $ \nu_{\pi_\theta}$ is the state-action visitation measure defined in \eqref{eq:def_visitation} with $\pi = \pi_\theta$, and $\cI(\theta)$ is the Fisher information defined in \eqref{eq:fisher}.
\end{proposition}
\begin{proof}
	See \S\ref{sec:pf_prop_npg_form} for a detailed proof.
\end{proof} 
Note that the expression of the policy gradient $\nabla_\theta L(\theta, \beta)$ in \eqref{eq:pg_form} of Proposition \ref{prop:npg_form} involves the state-action value function $Q^{\pi_\theta}_{r_\beta}(s, a)$. To this end, we estimate the state-action value function $Q_r^\pi(s, a)$ by $\hat Q_{\omega}(s, a)$, which is parameterized as follows,
\begin{align}
\label{eq:para_q}
\hat Q_\omega(s, a) = u_\omega(s, a).
\end{align}
Here $u_\omega(s, a)$ is the neural network defined in \eqref{eq:def_nn} with $W=\omega$. In the sequel, we call $\omega$ the value parameter.

\subsection{GAIL with Alternating Updates} \label{sec:ac}
 We employ an actor-critic scheme with alternating updates of the policy and the reward function, which is presented in Algorithm \ref{alg:ac_gail}.
 Specifically, we update the policy parameter $\theta$ via natural policy gradient and update the reward parameter $\beta$ via gradient ascent in the actor step, while we estimate the state-action value function $Q^{\pi}_{r}(s, a)$ via neural temporal difference (TD) \citep{cai2019neural} in the critic step.

\vskip4pt

\noindent {\bf Actor Step.} In the $k$-th actor step, we update the policy parameter $\theta$ and the reward parameter $\beta$ as follows,
\begin{align}
\theta_{k+1} &= \tau_{k+1}^{-1} \cdot ( \tau_k \cdot \theta_k - \eta \cdot \delta_k  ), \label{eq:update_theta}\\
\beta_{k+1} &= \proj_{S_{B_\beta}}\bigl\{\beta_k + \eta \cdot \hat \nabla_\beta L(\theta_k, \beta_k)\bigr\},
\label{eq:update_beta}
\end{align}
where
\begin{align}
\tau_{k+1} &= \eta + \tau_k, \quad \delta_k \in \argmin_{\delta \in S_{B_\theta}} \norm[\big]{ \hat \cI(\theta_k) \delta - \tau_k \cdot \hat \nabla_\theta L(\theta_k, \beta_k) }_2. \label{eq:def_delta} 
\end{align}
Here $\eta>0$ is the stepsize, $S_{B_\theta}$ and $S_{B_\beta}$ are the parameter domains of $\theta$ and $\beta$ defined in \eqref{eq:def_s_b} with domain radii $B_\theta$ and $B_\beta$, respectively, $\proj_{S_{B_\beta}}\!: \RR^{md} \rightarrow S_{B_\beta}$ is the projection operator, $\tau_k$ is the inverse temperature parameter of $\pi_{\theta_k}$, and $\hat \cI(\theta_k), \hat \nabla_\theta L(\theta_k, \beta_k), \hat \nabla_\beta L(\theta_k, \beta_k)$ are the estimators of $\cI(\theta_k), \nabla_\theta L(\theta_k, \beta_k), \nabla_\beta L(\theta_k, \beta_k)$, respectively, which are defined in the sequel. In \eqref{eq:update_theta}, we update the policy parameter $\theta_k$ along the direction $\delta_k$, which approximates the natural policy gradient $\cI(\theta)^{-1}\cdot \nabla_\theta L(\theta, \beta)$, and in \eqref{eq:def_delta} we update the inverse temperature parameter $\tau_k$ to ensure that $\theta_{k+1} \in S_{B_\theta}$. Meanwhile, in \eqref{eq:update_beta}, we update the reward parameter $\beta$ via (projected) gradient ascent.
Following from $\eqref{eq:fisher_form}$ and \eqref{eq:pg_form} of Proposition \ref{prop:npg_form}, we construct the estimators of $\cI(\theta_k)$ and $\nabla_\theta L(\theta_k, \beta_k)$ as follows,
\begin{align}
\label{eq:est_fisher}
\hat \cI(\theta_k) &= \frac{\tau_k^2}{N} \sum_{i=1}^{N} \phic_{\theta_k}(s_i, a_i) \, \phic_{\theta_k} (s_i, a_i)^\top,\\
\label{eq:est_pg}
\hat \nabla_\theta L(\theta_k, \beta_k) &= - \frac{\tau_k}{N} \sum_{i=1}^N \hat Q_{\omega_k} (s_i, a_i) \cdot \phic_{\theta_k}(s_i, a_i), 
\end{align}
where $\{(s_i, a_i)\}_{i\in[N]}$ is sampled from the state-action visitation measure $\nu_{ \pi_{\theta_k}}$ given $\theta_k$ with the batch size $N$, and $\hat Q_{\omega_k}(s, a)$ is the estimator of $Q^{\pi_{\theta_k}}_{r_{\beta_k}}(s, a)$ computed in the critic step. Meanwhile, following from \eqref{eq:rg_form} of Proposition \ref{prop:npg_form}, we construct the estimator of $\nabla_\beta L(\theta_k, \beta_k)$ as follows,
\begin{align}
\label{eq:est_rg}
\hat \nabla_\beta L(\theta, \beta) = \frac{1}{N\cdot (1-\gamma)} \sum_{i=1}^{N} \bigl[ \phi_{\beta_k}(s_i^\rE, a_i^\rE)  - \phi_{\beta_k}(s_i, a_i)  \bigr] - \lambda\cdot \nabla_\beta \psi(\beta_k),
\end{align}
where $\{(s_i^\rE, a_i^\rE)\}_{i\in[N]}$ is the expert trajectory. For notational simplicity, we write $\pi_k = \pi_{\theta_k}$, $r_k(s, a) = r_{\beta_k}(s, a)$, $d_k= d_{\pi_k}$ and $ \nu_k = \nu_{\pi_k}$ for the $k$-th step hereafter, where $\pi_\theta$ is the policy, $r_\beta(s, a)$ is the reward function, and $d_\pi, \nu_\pi$ are the visitation measures defined in \eqref{eq:def_visitation}.  

%

\vskip4pt
\noindent {\bf Critic Step. }Note that the estimator $\hat \nabla_\theta L(\theta, \beta)$ in \eqref{eq:est_pg} involves the estimator $\hat Q_{\omega_k}(s, a)$ of $Q^{\pi_k}_{r_k}(s, a)$. To this end, we parameterize $\hat Q_\omega(s, a) $ as in \eqref{eq:para_q} and adapt neural TD \citep{cai2019neural}, which solves the following minimization problem,
\begin{align}
\label{eq:def_tdlearn}
\omega_k = \argmin_{\omega\in S_{B_\omega}} \EE_{(s, a)\sim \rho_k} \bigl[ \hat Q_\omega(s, a) - \cT^{\pi_k}_{r_k} \hat Q_\omega (s, a) \bigr]^2.
\end{align}
Here $S_{B_\omega}$ is the parameter domain with domain radius $B_\omega$, $\rho_k$ is the state-action stationary distribution induced by $\pi _k$, and $\cT^{\pi_k}_{r_k}$ is the Bellman operator. Note that the Bellman operator $\cT^{\pi}_{r}$ is defined as follows,
\begin{align*}
\cT^\pi_r Q(s, a) = (1-\gamma)\cdot r(s, a) + \gamma \cdot \EE_\pi \bigl[ Q(s', a') \,\big|\, s, a \bigr],
\end{align*}
where the expectation is taken with respect to $s' \sim P(\cdot \given s, a)$ and $a'\sim \pi(\cdot \given s')$.
In neural TD, we iteratively update the value parameter $\omega$ via
\begin{align}
\label{eq:td_update}
\delta(j) &=  Q_{\omega(j)}(s, a) - r(s, a) - \gamma\cdot  Q_{\omega(j)}(s', a') , \nonumber \\
\omega(j+1)
&= \proj_{S_{B_\omega}} \bigl\{ \omega(j) - \alpha \cdot \delta(j) \cdot \nabla_\omega Q_{\omega(j)} (s, a) \bigr\},
\end{align}
where $\delta(j)$ is the Bellman residual, $\alpha>0$ is the stepsize,  $(s, a)$ is sampled from the state-action stationary distribution $\rho_k$, and $s' \sim P(\cdot\given s, a), a' \sim \pi_k(\cdot \given s')$ are the subsequent state and action. We defer the detailed discussion of neural TD to \S\ref{sec:td_learning}.

To approximately obtain the compatible function approximation \citep{sutton2000policy, wang2019neural}, we share the random initialization among the policy $\pi_\theta$, the reward function $r_\beta(s, a)$, and the state-action value function $\hat Q_{\omega}(s, a)$. In other words, we set $\theta_0 = \beta_0 = \omega(0) = W_0$ in our algorithm, where $W_0$ is the random initialization in \eqref{eq:init}. 
The output of GAIL is the mixed policy $\bar \pi$ \citep{altman1999constrained}. Specifically, the mixed policy $\bar \pi$ of $\pi_0, \ldots, \pi_{T-1}$ is executed by randomly selecting a policy $\pi_k$ for $k\in [0: T-1]$ with equal probability before time $t = 0$ and exclusively following $\pi_k$ thereafter. It then holds for any reward function $r(s, a)$ that 
\begin{align}
J(\bar \pi; r) = \frac{1}{T} \sum_{k=0}^{T-1} J(\pi_k; r). \label{eq:mixed_policy}
\end{align}

\begin{algorithm}[h]
	\caption{GAIL}
	\label{alg:ac_gail}
	\begin{algorithmic}[1]
		\REQUIRE Expert trajectory $\{(s_i^\rE, a_i^\rE)\}_{i \in [T_\rE]}$, number of iterations $T$, number of iterations $T_{\text{TD}}$ of neural TD, stepsize $\eta$, stepsize $\alpha$ of neural TD, batch size $N$, and domain radii $B_\theta, B_\omega, B_\beta$.
		\STATE {\bf Initialization.} Initialize $b_l\sim \unif(\{-1, 1\})$ and $[W_0 ]_l\sim N(0, I_d/d)$ for any $l\in[m]$ and set $\tau_0 \leftarrow 0$, $\theta_0 \leftarrow W_0$, and $\beta_0 \leftarrow W_0$.
		\FOR{$k = 0,1, \ldots, T-1$}
		\STATE Update value parameter $\omega_{k}$ via Algorithm \ref{alg:td_learn} with $\pi_{k}$, $r_{k}$, $W_0$, $b$, $T_{\td}$, and $\alpha$ as the input.
		\STATE Sample $\{(s_i, a_i)\}_{i=1}^N$ from the state-action visitation measure $\nu_k$, and estimate $\hat \cI(\theta_k)$, $\hat \nabla_\theta L(\theta_k, \beta_k)$, and $\hat \nabla_\beta L(\theta_k, \beta_k)$ via \eqref{eq:est_fisher}, \eqref{eq:est_pg},  and \eqref{eq:est_rg}, respectively.
		\STATE Solve $\delta_k \leftarrow \argmin_{\delta \in S_{\theta}} \bigl\| \hat \cI(\theta_k) \cdot \delta - \tau_k \cdot \hat \nabla_\theta L(\theta_k, \beta_k) \bigr\|_2$ and set $\tau_{k+1} \leftarrow \tau_k + \eta$.
		\STATE Update policy parameter $\theta$ via $\theta_{k+1} \leftarrow \tau_{k+1}^{-1}\cdot ( \tau_k\cdot \theta_k - \eta \cdot \delta_k)$. 
	    \STATE Update reward parameter $\beta$ via $ \beta_{k+1} \leftarrow \proj_{S_{B_\beta}} \{ \beta_k + \eta\cdot\hat \nabla_\beta L(\theta_k, \beta_k) \}$.
		\ENDFOR
		\ENSURE Mixed policy $\bar \pi$ of $\pi_{0}, \ldots, \pi_{T-1}$.
	\end{algorithmic}
\end{algorithm}


\section{Main Results} \label{sec:result}
In this section, we first present the assumptions for our analysis. Then, we establish the global optimality and convergence of Algorithm \ref{alg:ac_gail}.

\subsection{Assumptions} \label{sec:asp}
We impose the following assumptions on the stationary distributions $\varrho_k \in \cP(\cS), \rho_k\in \cP(\cS\times \cA)$ and the visitation measures $d_k\in\cP(\cS) , \nu_k \in \cP(\cS\times \cA)$.
\begin{assumption}
	\label{asp:measures}
	We assume that the following properties hold.
	\begin{itemize}
		\item[(a)] \label{asp:linear_nn} Let $\mu$ be either $\rho_k$ or $\nu_k$. We assume for an absolute constant $c > 0$ that
		\begin{align*}
		\EE_{(s, a) \sim \mu }\Bigl[ \ind\bigl\{ |W^\top (s, a) | \le y \bigr\} \Bigr] \le \frac{c \cdot y}{\norm{W}_2}, \quad \forall y>0, W \neq 0.
		\end{align*}
		
		\item[(b)] \label{asp:visitation} We assume for an absolute constant $C_h > 0$ that
		\begin{align*}
		&\max_{k\in \NN}\Biggl\{\norm[\bigg]{\frac{\rd d_\rE}{\rd d_k}}_{2, d_k} +  \norm[\bigg]{\frac{\rd \nu_\rE}{\rd \nu_k}}_{2, \nu_k} \Biggr\} \le C_h, \\
		&\max_{ k\in\NN} \Biggl\{\norm[\bigg]{\frac{\rd d_\rE}{\rd \varrho_k}}_{2, \varrho_k} + \norm[\bigg]{\frac{\rd \nu_\rE}{\rd \rho_k}}_{2, \rho_k} \Biggr\} \le C_h.
		\end{align*}
		Here $\rd d_\rE/\rd d_k$, $\rd \nu_\rE/\rd \nu_k$, $\rd d_\rE/\rd \varrho_k $, and $\rd \nu_\rE/\rd \rho_k$ are the Radon-Nikodym derivatives.
	\end{itemize}
\end{assumption}
Assumption \ref{asp:linear_nn}\,(a) holds when the probability density functions of $\rho_{k}$ and $\nu_{k}$ are uniformly upper bounded across $k$.
Assumption \ref{asp:visitation}\,(b) states that the concentrability coefficients are uniformly upper bounded across $k$, which is commonly used in the analysis of RL \citep{szepesvari2005finite, munos2008finite,antos2008learning,farahmand2010error,scherrer2015approximate,farahmand2016regularized,lazaric2010analysis}.

For notational simplicity, we write $u_0(s, a) = u_{W_0}(s, a)$ and $\phi_0(s, a) = \phi_{W_0}(s,a)$, where $u_{W_0}(s, a)$ is the neural network defined in \eqref{eq:def_nn} with $W=W_0$, $\phi_{W_0}(s, a)$ is the feature vector defined in \eqref{eq:def_feature} with $W = W_0$, and $W_0$ is the random initialization in \eqref{eq:init}. We impose the following assumptions on the neural network $u_0(s, a)$ and the transition kernel $P$.
\begin{assumption}
	\label{asp:nn}
	We assume that the following properties hold.
	\begin{itemize}
		\item[(a)]\label{asp:bound_init}  Let $\bar U = \sup_ {(s, a) \in \cS\times \cA}| u_0(s, a)|$. We assume for absolute constants $M_0>0$ and $v>0$ that 
		\begin{align}
		\label{eq:bound_init}
			\EE_\init[\bar U^2] \le M_0^2, \quad
			\PP(\bar U > t ) \le \exp(-v\cdot t^2), \quad \forall t > 2M_0.
		\end{align}
		\item[(b)]\label{asp:trans_class} We assume that the transition kernel $P$ belongs to the following class,
		\begin{align*}
		\tilde \cM_{\infty, B_P} = \biggl\{ P(s' \given s, a) = \int \vartheta(s, a;w) ^\top \varphi(s'; w) \,\rd q(w) \,\bigg|\,  \sup_{ w} \biggl\| \int\varphi(s; w) \rd s \biggr\|_2 \le B_P  \biggr\}.
		\end{align*}
		Here $B_P>0$ is an absolute constant, $q$ is the probability density function of $N(0, I_d/ d)$, and $\vartheta(s, a; w)$ is defined as
		$
		\vartheta(s, a; w) = \ind \{ w^\top (s, a) >0 \} \cdot (s, a).
		$
	\end{itemize}
\end{assumption}
Assumption \ref{asp:nn}\,(b) states that the MDP belongs to (a variant of) the class of linear MDPs \citep{yang2019sample, yang2019reinforcement, jin2019provably, cai2019provably}. However, our class of transition kernels is infinite-dimensional, and thus, captures a rich class of MDPs. To understand Assumption \ref{asp:nn}\,(a), recall that we initialize the neural network with $[W_0]_l \sim N(0, I_d / d)$ and $b_l \sim \unif(\{-1, 1\})$ for any $l \in [m]$. Thus, the neural network $u_{0}(s, a)$ defined in \eqref{eq:def_nn} with $W=W_0$ converges to a Gaussian process indexed by $(s, a) \in \cS\times \cA$ as $m$ goes to infinity. Following from the facts that the maximum of a Gaussian process over a compact index set is sub-Gaussian \citep{vandegeer2014onhigher} and that $\cS\times \cA$ is compact, it is reasonable to assume that $\sup_{(s, a) \in \cS\times \cA} | u_0(s, a)|$ is sub-Gaussian, which further implies the existence of the absolute constants $M_0$ and $v$ in \eqref{eq:bound_init} of Assumption \ref{asp:nn}\,(a). Moreover, if we assume that $m$ is even and initialize the parameters $W_0, b$ as follows,
\begin{align}
\label{eq:init-sym}
\begin{cases}
[W_0]_l \overset{\iid}{\sim} N(0, I_d / d), \quad b_l \sim \unif\bigl(\{-1, 1\}\bigr), \quad &\forall l =1, \ldots, m/2, \\
[W_0]_l = [W_0]_{l - m/2}, \quad  b_l = - b_{l - m/2}, \quad &\forall l = m/2 +1, \ldots, m,
\end{cases}
\end{align}
we have that $u_0(s, a) = 0$ for any $(s, a) \in \cS\times \cA$, which allows us to set $M_0 = 0$ and $v = + \infty$ in Assumption \ref{asp:bound_init}\,(a). Also, it holds that $0 = u_0(s, a) \in \cR_\beta$, which implies that $\DD_{\cR_\beta}(\pi_1, \pi_2) \ge 0$ for any $\pi_1$ and $\pi_2$.
The proof of our results with the random initialization in \eqref{eq:init-sym} is identical.

Finally, we impose the following assumption on the regularizer $\psi(\beta)$ and the variances of the estimators $ \hat \cI(\theta)$, $\hat \nabla_\theta L(\theta, \beta)$, and $\hat \nabla_\beta L(\theta, \beta)$ defined in \eqref{eq:est_fisher}, \eqref{eq:est_pg}, and \eqref{eq:est_rg}, respectively.
\begin{assumption}
	\label{asp:more}
	We assume that the following properties hold.
	\begin{itemize}
		\item[(a)]\label{asp:variance} We assume for an absolute constant $\sigma > 0$ that
		\begin{align}
		&\EE_k\bigg[ \norm[\Big]{ \hat \cI(\theta_k) W - \EE_k\bigl[ \hat \cI(\theta_k) W \bigr] }_2^2 \bigg] \le \tau_k^4\cdot \sigma^2 / N, \quad \forall W \in S_{B_\theta}, \label{eq:var1} \\
		&\EE_k \biggl[ \norm[\Big]{ \hat\nabla_\theta L(\theta_k, \beta_k) - \EE_k\bigl[\hat \nabla_\theta L(\theta_k, \beta_k)\bigr] }_2^2 \biggr] \le \tau_k^2 \cdot \sigma^2 / N, \label{eq:var2} \\
		&\EE_k \biggl[ \norm[\Big]{\hat \nabla_\beta L(\theta_k, \beta_k) - \EE_k\bigl[ \hat \nabla_\beta L(\theta_k, \beta_k) \bigr] }_2^2 \biggr] \le \sigma^2 / N, \label{eq:var3}
		\end{align}
		where $\tau_k$ is the inverse temperature parameter in \eqref{eq:para_pi}, $N\in\NN_+$ is the batch size, and $S_{B_\theta}$ is the parameter domain of $\theta$ defined in \eqref{eq:def_s_b} with the domain radius $B_\theta$. Here the expectation $\EE_k$ is taken with respect to the $k$-th batch, which is drawn from $\nu_k$ given $\theta_k$.
		
		\item[(b)]\label{asp:regularizer} We assume that the regularizer $\psi(\beta)$ in \eqref{eq:def_il} is convex and $L_\psi$-Lipschitz continuous over the compact parameter domain $S_{B_\beta}$. 
	\end{itemize}
\end{assumption}
Assumption \ref{asp:more}\,(a) holds when $\hat Q_{\omega_k}(s_i, a_i) \cdot \phic_{\theta_k}(s_i, a_i)$, $\phic_{\theta_k}(s_i, a_i) \phic_{\theta_k}(s_i, a_i)^\top$, and $\phi_{\beta_k}(s_i, a_i)$ have uniformly upper bounded variances across $i \in [m]$ and $k$, and the Markov chain that generates $\{(s_i, a_i)\}_{i\in [N]}$ mixes sufficiently fast \citep{zhang2019global}. Similar assumptions are also used in the analysis of policy optimization \citep{xu2019improved, xu2019sample}. Also, Assumption \ref{asp:more}\,(b) holds for most commonly used regularizers \citep{ho2016generative}.


\subsection{Global Optimality and Convergence} \label{sec:conv}
In this section, we establish the global optimality and convergence of Algorithm \ref{alg:ac_gail}. 
The following proposition adapted from \cite{cai2019neural} characterizes the global optimality and convergence of neural TD, which is presented in Algorithm \ref{alg:td_learn}.
\begin{proposition}[Global Optimality and Convergence of Neural TD]
	\label{prop:conv_td}
	 In Algorithm \ref{alg:td_learn}, we set $T_{\TD} = \Omega(m)$, $\alpha= \min\{ (1-\gamma) / 8, m^{-1/2} \}$, and $B_\omega = c \cdot (B_\beta + B_P\cdot (M_0 + B_\beta)) $ for an absolute constant $c > 0$. Let $\pi_k, r_k$ be the input and $\omega_k $ be the output of Algorithm \ref{alg:td_learn}. 
	Under Assumptions \ref{asp:measures} and \ref{asp:nn}, it holds for an absolute constant $C_v > 0$ that
	\begin{align}
	\label{eq:def_errq0}
	\EE_\init\Bigl[ \norm[\big]{Q_{\omega_k} (s, a) -Q^{\pi_k}_{r_k}(s, a) }^2_{2,\rho_k} \Bigr] = \cO\bigl(B_\omega^3 \cdot  m^{-1/2} + B_\omega^{5/2} \cdot m^{-1/4} + B_\omega^2\cdot \exp(-C_v\cdot B_\omega^2)\bigr).
	\end{align}
	Here the expectation $\EE_\init$ is taken with respect to the random initialization in \eqref{eq:init}.
\end{proposition}
\begin{proof}
	See \S\ref{sec:pf_conv_td} for a detailed proof.
\end{proof}
The term $ B_\omega^2\cdot \exp(-C_v\cdot B_\omega^2) $ in \eqref{eq:def_errq0} of Proposition \ref{prop:conv_td} characterizes the hardness of estimating the state-action value function $Q^{\pi_k}_{r_k}(s, a)$ by the neural network defined in \eqref{eq:def_nn}, which arises because $\|Q^{\pi_k}_{r_k}(s, a)\|_\infty$ is not uniformly upper bounded across $k$. Note that if we employ the random initialization in \eqref{eq:init-sym}, we have that $C_v = +\infty$. And consequently, such a term vanishes.
We are now ready to establish the global optimality and convergence of Algorithm \ref{alg:ac_gail}.
\begin{theorem}[Global Optimality and Convergence of GAIL]
	\label{th:convergence}
	We set $\eta = 1 / \sqrt{T}$ and $B_{\omega} =  c \cdot ( B_\beta + B_P\cdot (M_0 + B_\beta) ) $ for an absolute constant $c > 0$, and $B_{\theta} = B_{\omega}$ in Algorithm \ref{alg:ac_gail}. Let $\bar\pi$ be the output of Algorithm \ref{alg:ac_gail}. Under Assumptions \ref{asp:measures}-\ref{asp:more}, it holds that
	\begin{align}
	\label{eq:convergence}
	\EE\bigl[\DD_{\cR_{\beta}}(\pi_\rE, \bar \pi)\bigr]
	&\le  \underbrace{\frac{ (1-\gamma)^{-1}\cdot \log | \cA | + 13 \bar B^2 + M_0^2 + 8}{\sqrt{T}}}_{\displaystyle{\text{(i)}}} + \underbrace{2\lambda\cdot L_\psi\cdot \bar B}_{\displaystyle{\text{(ii)}}} + \underbrace{\frac{1}{T}{\sum_{k=0}^{T-1}\varepsilon_k}}_{\displaystyle{\text{(iii)}}}.
	\end{align}
	Here $\bar B = \max\{B_\theta, B_\omega, B_\beta\}$, $\DD_{\cR_\beta}$ is the $\cR_\beta$-distance defined in \eqref{eq:def_dis} with $\cR_\beta = \{ r_\beta(s, a) \given \beta \in S_{B_\beta} \}$, the expectation is taken with respect to the random initialization in \eqref{eq:init} and the $T$ batches, and the error term $\varepsilon_k$ satisfies that
	\begin{align}
	\label{eq:def_vareps}
	&\varepsilon_k = \underbrace{ 2\sqrt{2}\cdot C_h \cdot \bar B \cdot \sigma\cdot N^{-1/2}}_{\displaystyle{\text{(iii.a)}}} + \underbrace{\err_{Q, k}}_{\displaystyle\text{(iii.b)}}  +  \underbrace{\cO(k\cdot  \bar B^{3/2} \cdot m^{-1/4} + \bar B^{5/4} \cdot m^{-1/8} )}_{\displaystyle\text{(iii.c)}},
	\end{align}
	where $C_h$ is defined in Assumption \ref{asp:measures}, $L_\psi$ and $\sigma$ are defined in Assumption \ref{asp:regularizer}, and $\err_{Q, k} = \cO(B_\omega^3 \cdot m^{-1/2} + B_\omega^{5/2}\cdot m^{-1/4} + B_\omega^2\cdot \exp(-C_v\cdot B_\omega^2))$ is the error induced by neural TD (Algorithm \ref{alg:td_learn}).
	
\end{theorem}
\begin{proof}
	See \S\ref{sec:analysis} for a detailed proof.
\end{proof}
The optimality gap in \eqref{eq:convergence} of Theorem \ref{th:convergence} is measured by the expected $\cR_\beta$-distance $\DD_{\cR_{\beta}}(\pi_{\rE}, \bar \pi)$ between the expert policy $\pi_\rE$ and the learned policy $\bar \pi$. Thus, by showing that the optimality gap is upper bounded by $ \cO(1/\sqrt{T})$, we prove that $\bar \pi$ (approximately) outperforms the expert policy $\pi_\rE$ in expectation when the number of iterations $T$ goes to infinity.
As shown in \eqref{eq:convergence} of Theorem \ref{th:convergence}, the optimality gap is upper bounded by the sum of the three terms. Term (i) corresponds to the $1/\sqrt{T}$ rate of convergence of Algorithm \ref{alg:ac_gail}. Term (ii) corresponds to the bias induced by the regularizer $\lambda\cdot \psi(\beta)$ in the objective function $L(\theta, \beta)$ defined in \eqref{eq:def_il}. Term (iii)  is upper bounded by the sum of the three terms in \eqref{eq:def_vareps} of Theorem \ref{th:convergence}. In detail, term (iii.a) corresponds to the error induced by the variances of $\hat \cI(\theta)$, $\hat \nabla_\theta L(\theta, \beta)$, and $\hat \nabla_\beta L(\theta, \beta)$ defined in \eqref{eq:var1}, \eqref{eq:var2}, and \eqref{eq:var3} of Assumption \ref{asp:more}, which vanishes as the batch size $N$ in Algorithm \ref{alg:ac_gail} goes to infinity. Term (iii.b) is the error of estimating $Q^\pi_r(s, a)$ by $\hat Q_\omega(s, a)$ using neural TD (Algorithm \ref{alg:td_learn}). As shown in Proposition \ref{prop:conv_td}, the estimation error $\epsilon_{Q, k}$ vanishes as $m$ and $B_\omega$ go to infinity. Term (iii.c) corresponds to the linearization error of the neural network defined in \eqref{eq:def_nn}, which is characterized in Lemma \ref{lem:linear_nn}.
Following from Theorem \ref{th:convergence}, it holds for $B_\omega = \Omega((C_v^{-1}\cdot \log T)^{1/2})$, $m = \Omega(\bar B^{10}\cdot T^{6})$, and $N = \Omega(\bar B^2 \cdot T \cdot \sigma^2)$ that $\EE\bigl[\DD_{\cR_{\beta}}(\pi_\rE, \bar \pi)\bigr] = \cO(T^{-1/2} + \lambda)$, which implies the $1/ \sqrt{T}$ rate of convergence of Algorithm \ref{alg:ac_gail} (up to the bias induced by the regularizer).


\section{Proof of Main Results}\label{sec:analysis}

In this section, we present the proof of Theorem \ref{th:convergence}, which establishes the global optimality and convergence of Algorithm \ref{alg:ac_gail}. For notational simplicity,  we write $\pi^s(a) = \pi(a\given s)$ for any policy $\pi$,  state $s\in\cS$, and action $a \in \cA$. For any policies $\pi_1,\pi_2$ and distribution $\mu$ over $\cS$, we denote the expected Kullback-Leibler (KL) divergence by $\kl^{\mu}$, which is defined as $\kl^\mu(\pi_1 \,\|\, \pi_2) = \EE_{s\sim\mu} [ \kl(\pi_1^s \,\|\, \pi_2^s)]$.
 For any visitation measures $d_\pi \in \cP(\cS)$ and $\nu_\pi \in \cP(\cS\times \cA)$, we denote by $\EE_{d_\pi}$ and $\EE_{\nu_\pi}$ the expectations taken with respect to $s\sim d_\pi$ and $(s, a)\sim \nu_\pi$, respectively.

Following from the property of the mixed policy $\bar \pi$ in \eqref{eq:mixed_policy}, we have that
\begin{align}
\label{eq:gap1}
\EE\bigl[\DD_{\cR_{\beta}}(\pi_\rE, \bar \pi)\bigr] &= \EE\bigl[ \max_{\beta'\in S_{B_\beta}} J(\pi_\rE; r_{\beta'}) - J(\bar \pi; r_{\beta'}) \bigr] \nonumber \\
&= \EE\biggl[ \max_{\beta'\in S_{B_\beta}} \frac{1}{T}\sum_{k=0}^{T-1}J(\pi_\rE; r_{\beta'}) - J(\pi_k; r_{\beta'}) \biggr].
\end{align}
We now upper bound the optimality gap in \eqref{eq:gap1} by upper bounding the following difference of expected cumulative rewards,
\begin{align}
\label{eq:def_potential}
& J(\pi_\rE; r_{\beta'}) - J(\pi_k; r_{\beta'}) = \underbrace{J(\pi_\rE; r_k) - J(\pi_k; r_k) }_{\displaystyle\text{(i)}} + \underbrace{ L(\theta_k, \beta') - L(\theta_k, \beta_k) }_{ \displaystyle\text{(ii)} } + \underbrace{\lambda\cdot \bigl(\psi(\beta') - \psi(\beta_k) \bigr)}_{\displaystyle\text{(iii)}},
\end{align}
where $\beta' \in S_{B_\beta} $ is chosen arbitrarily and $L(\theta, \beta)$ is the objective function defined in \eqref{eq:def_il}. Following from Assumption \ref{asp:more} and the fact that $\beta_k, \beta'\in S_{B_\beta}$, we have that
\begin{align}
\label{eq:track_iii}
\lambda\cdot\bigl(\psi(\beta') - \psi(\beta_k) \bigr) \le \lambda\cdot L_\psi \cdot \norm{\beta' - \beta_k}_2 \le \lambda\cdot  L_\psi \cdot 2B_\beta,
\end{align}
which upper bounds term (iii) of \eqref{eq:def_potential}.
It remains to upper bound terms (i) and (ii) of \eqref{eq:def_potential}, which hinges on the one-point convexity of $J(\pi; r)$ with respect to $\pi$ and the (approximate) convexity of $L(\theta, \beta)$ with respect to $\beta$.
\vskip4pt
\noindent{\bf Upper bound of term (i) in \eqref{eq:def_potential}.} In what follows, we upper bound term (i) of \eqref{eq:def_potential}. We first introduce the following cost difference lemma \citep{kakade2002approximately}, which corresponds to the one-point convexity of $J(\pi; r)$ with respect to $\pi$. Recall that $d_\rE \in \cP(\cS)$ is the state visitation measure induced by the expert policy $\pi_\rE$. 
\begin{lemma}[Cost Difference Lemma, Lemma 6.1 in \cite{kakade2002approximately}]
	\label{lem:cost_diff}
	For any policy $\pi$ and reward function $r(s, a)$, it holds that 
	\begin{align}
	\label{eq:cost_diff}
	J(\pi_\rE; r) - J(\pi; r) = (1-\gamma)^{-1}\cdot \EE_{d_\rE} \Bigl[ \bigl\langle Q^\pi_r(s, \cdot), \pi_\rE^s - \pi^s \bigr\rangle_\cA \Bigr],
	\end{align}
	where $\gamma$ is the discount factor.
\end{lemma}
Furthermore, we establish the following lemma, which upper bounds the right-hand side of \eqref{eq:cost_diff} in Lemma \ref{lem:cost_diff}.
\begin{lemma}
	\label{lem:pe0}
	Under Assumptions \ref{asp:measures}-\ref{asp:more}, we have that
	\begin{align*}
	\EE_{d_\rE}\Bigl[\bigl\langle Q^{\pi_k}_{r_k}(s, \cdot), \pi_\rE^s - {\pi_k^s} \bigr\rangle_\cA\Bigr] = \eta^{-1}\cdot \kl^{d_\rE}(\pi_\rE\,\|\, \pi_k ) -\eta^{-1}\cdot \kl^{d_\rE}(\pi_\rE\,\|\, \pi_{k+1} ) + \Delta_k^{\text{(i)}},
	\end{align*}
	where 
	\begin{align}
	\label{eq:pe0}
	\EE\bigl[|\Delta_k^{\text{(i)}}|\bigr] &= 2\sqrt{2}\cdot C_h \cdot B_\theta^{1/2} \cdot \sigma^{1/2}\cdot N^{-1/4}    + \err_{Q, k} +  \eta \cdot  (M_0^2 + 9B_\theta^2) \nonumber \\
	& \quad +  \cO(\eta^{-1}\cdot \tau_{k+1} \cdot B_\theta^{3/2} \cdot m^{-1/4} + B_\theta^{5/4} \cdot m^{-1/8} ).
	\end{align}
	Here $M_0$ is defined in Assumption \ref{asp:bound_init}, $\sigma$ is defined in Assumption \ref{asp:more}, $N$ is the batch size in \eqref{eq:est_fisher}-\eqref{eq:est_rg}, and $\err_{Q, k} =\cO(B_\omega^3 \cdot m^{-1/2} + B_\omega^{5/2}\cdot m^{-1/4} + B_\omega^2\cdot \exp(-C_v\cdot B_\omega^2))$ for an absolute constant $C_v>0$, which depends on the absolute constant $v$ in Assumption \ref{asp:bound_init}.
\end{lemma}
\begin{proof}
	See \S\ref{sec:pf_lem_pe0} for a detailed proof.
\end{proof}
Combining Lemmas \ref{lem:cost_diff} and \ref{lem:pe0}, we have that
\begin{align}
\label{eq:track_i}
J(\pi_\rE; r_k) - J(\pi_k; r_k)  \le \frac{\kl^{d_\rE}(\pi_\rE\,\|\, \pi_k ) - \kl^{d_\rE}(\pi_\rE\,\|\, \pi_{k+1} ) + \eta \cdot  \Delta_k^{\text{(i)}}}{\eta\cdot (1-\gamma)} ,
\end{align}
which upper bounds term (i) of \eqref{eq:def_potential}. Here $\Delta_k^{\text{(i)}}$ is upper bounded in \eqref{eq:pe0} of Lemma \ref{lem:pe0}.

\vskip4pt

\noindent{\bf Upper bound of term (ii) in \eqref{eq:def_potential}.} In what follows, we upper bound term (ii) of \eqref{eq:def_potential}.
We first establish the following lemma, which characterizes the (approximate) convexity of $L(\theta, \beta)$ with respect to $\beta$.
\begin{lemma}
	\label{lem:reward}
	Under Assumption \ref{asp:linear_nn}, it holds for any $\beta'\in S_{B_\beta}$ that
	\begin{align*}
	\EE_\init\bigl[  L(\theta_k, \beta') - L(\theta_k, \beta_k)  \bigr] = \EE_\init\bigl[\nabla_\beta L(\theta_k ,\beta_k)^\top (\beta' - \beta_k) \bigr]
	 + \cO( B_\beta ^{3/2}\cdot m^{-1/4}).
	\end{align*} 
\end{lemma}
\begin{proof}
	See \S\ref{sec:pf_lem_reward} for a detailed proof.
\end{proof}
The term $\cO( B_\beta ^{3/2}\cdot m^{-1/4})$ in Lemma \ref{lem:reward} arises from the linearization error of the neural network, which is characterized in Lemma \ref{lem:linear_nn}.
Based on Lemma \ref{lem:reward}, we establish the following lemma, which upper bounds term (ii) of \eqref{eq:def_potential}.
\begin{lemma}
	\label{lem:re0}
	Under Assumptions \ref{asp:linear_nn} and \ref{asp:more}, we have that 
	\begin{align*}
	L(\theta_k, \beta') - L(\theta_k, \beta_k) \le \eta^{-1}\cdot \norm{\beta_k - \beta'}_2^2 - \eta^{-1}\cdot  \norm{\beta_{k+1} - \beta'}_2^2 - \eta^{-1}\cdot  \norm{\beta_{k+1} - \beta_k}_2^2 + \Delta_k^{\text{(ii)}},
	\end{align*}
	where 
	\begin{align}
	\label{eq:re0}
	\EE\bigl[|\Delta_k^{\text{(ii)}}|\bigr] = \eta \cdot \bigl( (2+\lambda \cdot L_\psi)^2 + \sigma^2 \cdot  N^{-1} \bigr) + 2 B_\beta \cdot \sigma\cdot N^{-1/2} + \cO(  B_\beta^{3/2}\cdot m^{-1/4}).
	\end{align}

\end{lemma}
\begin{proof}
	See \S\ref{sec:pf_lem_re0} for a detailed proof.
\end{proof}
By Lemma \ref{lem:re0}, we have that
\begin{align}
\label{eq:track_ii}
L(\theta_k, \beta') - L(\theta_k, \beta_k) \le \eta^{-1}\cdot \bigl(\norm{\beta_k - \beta'}_2^2 - \norm{\beta_{k+1} - \beta'}_2^2 - \norm{\beta_{k+1} - \beta_k}_2^2 \bigr) + \Delta_k^{\text{(ii)}}  ,
\end{align} 
which upper bounds term (ii) of \eqref{eq:def_potential}. Here $\Delta_k^{\text{(ii)}}$ is upper bounded in \eqref{eq:re0} of Lemma \ref{lem:re0}.

\vskip4pt

Plugging \eqref{eq:track_iii}, \eqref{eq:track_i}, and \eqref{eq:track_ii} into \eqref{eq:def_potential}, we obtain that
\begin{align}
\label{eq:potential}
& J(\pi_\rE; r_{\beta'}) - J(\pi_k; r_{\beta'})  \\
&\quad \le  \frac{\kl^{d_\rE}(\pi_\rE\,\|\, \pi_k ) - \kl^{d_\rE}(\pi_\rE\,\|\, \pi_{k+1} )}{\eta\cdot(1-\gamma)} + \eta^{-1} \cdot \bigl( \norm{\beta_k - \beta'}_2^2  - \norm{\beta_{k+1} - \beta'}_2^2\bigr) + 2\lambda\cdot L_\psi\cdot B_\beta + \Delta_k. \nonumber
\end{align}
Here $\Delta_k = \Delta_k^{\text{(i)}} + \Delta_k^{\text{(ii)}}$, where $\Delta_k^{\text{(i)}}$ and $\Delta_k^{\text{(ii)}}$ are upper bounded in \eqref{eq:pe0} and \eqref{eq:re0} of Lemmas \ref{lem:pe0} and \ref{lem:re0}, respectively. Note that the upper bound of $\Delta_k$ does not depend on $\theta$ and $\beta$. Upon telescoping \eqref{eq:potential} with respect to $k$, we obtain that 
\begin{align}
\label{eq:telescope}
 J(\pi_\rE; r_{\beta'}) - J(\bar \pi; r_{\beta'})  & = \frac{1}{T} \sum_{k=0}^{T-1}\bigl[ J(\pi_\rE; r_{\beta'}) - J(\pi_k; r_{\beta'}) \bigr]  \\
& \le  \frac{(1-\gamma)^{-1}\cdot \kl^{d_\rE}(\pi_\rE\,\|\, \pi_0 ) + \norm{\beta_0 - \beta'}_2^2}{ \eta\cdot T} + 2\lambda \cdot L_\psi\cdot B_\beta + \frac{1}{T} \sum_{k=0}^{T-1} |\Delta_k|. \nonumber
\end{align}
Following from the fact that $\tau_0 = 0$ and the parameterization of $\pi_\theta$ in \eqref{eq:para_pi}, it holds that $\pi_0^s$ is the uniform distribution over $\cA$ for any $s\in \cS$. Thus, we have $\kl^{d_\rE}(\pi_\rE\,\|\, \pi_0 ) \le \log |\cA|$. Meanwhile, following from the fact that $\beta' \in S_{B_\beta}$, it holds that $\norm{\beta' - \beta_0}_2 \le B_\beta$. Finally, by setting $\eta = T^{-1/2}$, $\tau_k = k\cdot \eta$, and $\bar B = \max\{B_\theta, B_\beta\}$ in \eqref{eq:telescope}, we have that
\begin{align*}
\EE\bigl[\DD_{\cR_{\beta}}(\pi_\rE, \bar \pi)\bigr] &= \EE\bigl[ \max_{\beta' \in S_{B_\beta}}J(\pi_\rE; r_{\beta'}) - J(\bar \pi; r_{\beta'}) \bigr]\nonumber \\
& \le \frac{(1-\gamma)^{-1} \cdot \log|\cA| + 4B_\beta^2}{ \eta\cdot T} + 2\lambda\cdot L_\psi\cdot B_\beta  + \frac{\EE\bigl[ \max_{\beta'} \sum_{k=0}^{T-1} |\Delta_k|\bigr]}{T}\nonumber\\
 & = \frac{(1-\gamma)^{-1}\cdot \log | \cA | + 13 \bar B^2 + M_0^2 + 8 }{ \sqrt{T}} + 2\lambda\cdot L_\psi\cdot \bar B + \frac{\sum_{k=0}^{T-1}\varepsilon_k}{T}.
\end{align*}
Here $\varepsilon_k$ is upper bounded as follows,
\begin{align*}
&\varepsilon_k = 2\sqrt{2}\cdot C_h \cdot \bar B \cdot \sigma\cdot N^{-1/2} + \err_{Q, k} +  \cO(k \cdot  \bar B^{3/2} \cdot m^{-1/4} + \bar B^{5/4} \cdot m^{-1/8}),
\end{align*}
where $\err_{Q, k} = \cO(B_\omega^3 \cdot  m^{-1/2} + B_\omega^{5/2} \cdot m^{-1/4} + B_\omega^2\cdot \exp(-C_v\cdot B_\omega^2))$ for an absolute constant $C_v >0$.
Thus, we complete the proof of Theorem \ref{th:convergence}.

\bibliographystyle{ims}
\bibliography{rl_ref}

\newpage
\appendix 
\section{Neural Networks}
\label{sec:linear_nn}
In what follows, we present the properties of the neural network defined in \eqref{eq:def_nn}. First, we define the following function class.
\begin{definition}[Function Class]
	\label{def:func_class}
	For $B>0$ and $m \in \NN_+$, we define 
	\begin{align*}
	\cF_{B, m} = \bigl\{ W^\top \phi_0(s, a) \,\big|\, W\in \RR^{md},~ \|W - W_0\|_2 \le B \bigr\},
	\end{align*}
	where $\phi_0(s,a)$ is the feature vector defined in \eqref{eq:def_feature} with $W = W_0$.
\end{definition} 
As shown in \cite{rahimi2008random}, the feature $\phi_0(s, a)$ induces a reproducing kernel Hilbert space (RKHS), namely $\cH$. When $m$ goes to infinity, $\cF_{B, m}$ approximates a ball in $\cH$, which captures a rich class of functions \citep{hofmann2008kernel,rahimi2008random}. Furthermore, we obtain the following lemma from \cite{cai2019neural}, which characterizes the linearization error of the neural network defined in \eqref{eq:def_nn}.
\begin{lemma}[Linearization Error, Lemma 5.1 in \cite{cai2019neural}]
	\label{lem:linear_nn}
	Under Assumption~\ref{asp:linear_nn}, it holds for any $W, W_1, W_2 \in S_B$ that, 
	\begin{align*}
	&\EE_\init \Bigl[ \bigl\| W^\top \phi_{W_1}(s, a) - W^\top \phi_{W_2}(s, a) \bigr \|_{2, \mu}^2 \Bigr] = \cO(B^3 \cdot m^{-1/2}), \\
	&\EE_\init \Bigl[ \norm[\big]{W^\top \phi_{W_1}(s, a) - W^\top \phi_{W_2}(s, a)}_{1,\mu} \Bigr] = \cO(B^{3/2} \cdot m^{-1/4}),
	\end{align*} 
	where $\phi_W(s, a)$ is the feature vector defined in \eqref{eq:def_feature} and $\mu\in\cP(\cS\times \cA)$ is a distribution that satisfies Assumption \ref{asp:linear_nn}. 
\end{lemma}
\begin{proof}
	See \S\ref{sec:pf_lem_linear_nn} for a detailed proof.
\end{proof}
Following from Lemma \ref{lem:linear_nn}, the function class $\cF_{B, m}$ defined in Definition \ref{def:func_class} is a first-order approximation of the class of the neural networks defined in \eqref{eq:def_nn}. Meanwhile, we establish the following lemma to characterize the sub-Gaussian property of the neural network defined in \eqref{eq:def_nn}.
\begin{lemma}
	\label{lem:bound_nn}
	Under Assumption \ref{asp:nn}, for any $W, W'\in S_B$, it holds that $\sup_{(s, a) \in \cS\times \cA} |W^\top \phi_{W'}(s, a)|$ is sub-Gaussian, where the randomness comes from the random initialization $W_0$ in the definition of $S_B$ in \eqref{eq:def_s_b}. Moreover, it holds that
	\begin{align*}
	\EE_\init \Bigl[\sup_{(s,a) \in \cS\times \cA}\bigl| W^\top \phi_{W'}(s, a) \bigr|^2 \Bigr]  \le 2M_0^2 + 18B^2
	\end{align*}
	and that
	\begin{align*}
	\PP\Bigl( \sup_{(s,a) \in \cS\times \cA}\bigl| W^\top \phi_{W'}(s, a) \bigr| > t \Bigr)\le  \exp(-v \cdot t^2 / 2), \quad \forall t > 2M_0 + 6B.
	\end{align*}
\end{lemma}
\begin{proof}
	See \S\ref{sec:pf_lem_bound_nn} for a detailed proof.
\end{proof}

\subsection{Proof of Lemma \ref{lem:linear_nn}}
\label{sec:pf_lem_linear_nn}
\begin{proof}
	We consider any $W, W' \in S_B$. By the definition of $\phi_W(s, a)$ in \eqref{eq:def_feature} and the triangle inequality, we have that
	\begin{align}
	\label{eq:ln-1}
	&\bigl|W^\top \phi_{W'}(s, a) - W^\top \phi_0(s, a)\bigr | \nonumber\\ &\quad\le \frac{1}{\sqrt{m}} \sum_{l=1}^{m} \bigl| [W]_l^\top(s, a) \bigr| \cdot\Bigl| \ind\bigl\{(s,a)^\top[W']_l>0\bigr\}- \ind\bigl\{(s,a)^\top[W_0]_l>0\bigr\}\Bigr|.  
	\end{align}
	We now upper bound the right-hand side of \eqref{eq:ln-1}. For the term $| [W]_l^\top(s, a) |$ in \eqref{eq:ln-1}, we have that
	\begin{align}
	\label{eq:ln-2}
	\bigl| [W]_l^\top(s, a) \bigr| &\le \bigl| [W_0]_l^\top(s, a) \bigr| + \Bigl| \bigl([W]_l - [W_0]_l\bigr)^\top(s, a) \Bigr|\nonumber\\
	&\le \bigl| [W_0]_l^\top(s, a) \bigr| + \norm[\big]{[W]_l - [W_0]_l}_2,
	\end{align}
	where the first inequality follows from the triangle inequality and the second inequality follows from the Cauchy-Schwartz inequality and the fact that $\norm{(s, a)}_2\le 1$. To upper bound the term $ | \ind\{(s,a)^\top[W']_l>0\}- \ind\{(s,a)^\top[W_0]_l>0\}| $ on the right-hand side of \eqref{eq:ln-1}, note that $ \ind\{(s,a)^\top[W']_l>0\}\neq \ind\{(s,a)^\top[W_0]_l>0\}$ implies that 
	\begin{align*}
	\bigl|[W_0]_l^\top (s, a)\bigr| \le \bigl| [W']_l^\top (s, a) - [W_0]_l^\top (s, a) \bigr| \le \norm[\big]{[W']_l - [W_0]_l}_2.
	\end{align*}
	Thus, we have that 
	\begin{align}
	\label{eq:ln-4}
	\Bigl| \ind\bigl\{(s,a)^\top[W']_l>0\bigr\}- \ind\bigl\{(s,a)^\top[W_0]_l>0\bigr\}\Bigr| \le \ind\Bigl\{\bigl|(s,a)^\top[W_0]_l\bigr|\le \bigl\|[W']_l-[W_0]_l\bigr\|_2 \Bigr\}.
	\end{align}
	Plugging \eqref{eq:ln-2} and \eqref{eq:ln-4} into \eqref{eq:ln-1}, we have that
	\begin{align*}
	&\bigl|W^\top \phi_{W'}(s, a) - W^\top \phi_0(s, a)\bigr |\nonumber\\
	&\quad \le 
	\frac{1}{\sqrt{m}}\sum_{l=1}^m 
	\ind\Bigl\{\bigl|(s,a)^\top[W_0]_l\bigr|\le \bigl\|[W']_l-[W_0]_l\bigr\|_2 \Bigr\} \cdot \Bigl(\bigl|(s,a)^\top[W_0]_l\bigr|
	+ \bigl\|[W]_l-[W_0]_l \bigr\|_2
	\Bigr)\nonumber\\
	&\quad \le 
	\frac{1}{\sqrt{m}}\sum_{l=1}^m 
	\ind\Big\{\big|(s,a)^\top[W_0]_l\big|\le \big\|[W']_l-[W_0]_l\big\|_2 \Big\} \cdot \Bigl(\big\|[W']_l-[W_0]_l \big\|_2
	+\big\|[W]_l -[W_0]_l \big\|_2
	\Bigr).
	\end{align*}
	By the fact that $W, W' \in S_B$, we obtain that
	\begin{align*}
	\bigl| W^\top \phi_{W'}(s, a) - W^\top \phi_0(s, a) \bigr|^2\le \frac{4B^2}{m}\sum_{l=1}^m 
	\ind\Bigl\{\bigl|(s,a)^\top[W_0]_l\bigr|\le \bigl\|[W']_l-[W_0]_l \bigr\|_2 \Bigr\}.
	\end{align*}
    By setting $y = \|[W']_l-[W_0]_l \|_2$ in Assumption \ref{asp:linear_nn}, we have that
	\begin{align*}
	\bigl\| W^\top \phi_{W'}(s, a) - W^\top \phi_0(s, a) \bigr\|_{2, \mu}^2 \le \frac{8 B^2}{m}\sum_{l=1}^m 
	\frac{c\cdot  \bigl\|[W']_l -[W_0]_l\bigr\|_2}{\bigl\|[W_0]_l\bigr\|_2}.
	\end{align*}
	Taking the expectation with respect to the random initialization in \eqref{eq:init} and using the Cauchy-Schwartz inequality, we have that
	\begin{align*}
	&\EE_\init\Bigl[\big\| W^\top \phi_{W'}(s, a) - W^\top \phi_0(s, a) \big\|_{2, \mu}^2\Bigr] \nonumber \\
	&\quad \le 
	\EE_\init\biggl[\frac{8 cB^2}{m}
		\Bigl(\sum_{l=1}^m \bigl\|[W']_l-[W_0]_l\bigr\|_2^2\Bigr)^{1/2} \cdot  \Bigl(\sum_{l=1}^m 1/\bigl\|[W_0]_l\bigr\|_2^2\Bigr)^{1/2}
	\biggr]\\
	&\quad \le \frac{8cB^3}{m}\EE_\init\biggl[\Bigl(
		\sum_{l=1}^m 1/\bigl\|[W_0]_l\bigr\|_2^2
	\Bigr)^{1/2}\biggr]\\
	&\quad \le \frac{8 cB^3}{\sqrt{m}}\Bigl(\EE_{w\sim N(0,I_d/d)}\bigl[1/\|w\|_2^2\bigr]\Bigr)^{1/2}\\
	&\quad = \cO(B^3\cdot m^{-1/2}),
	\end{align*}
	where the second inequality follows from the fact that $\norm{W' - W_0}_2 \le B$, the third inequality follows from Jensen's inequality, and the last inequality follows from Assumption \ref{asp:linear_nn} and Lemma \ref{lem:linear_nn}.
	Thus, for any $W, W_1, W_2 \in S_B$, we have that
	\begin{align*}
	&\EE_\init\Bigl[\big\| W^\top \phi_{W_1}(s, a) - W^\top \phi_{W_2}(s, a) \big\|_{2,  \mu}^2\Bigr] \nonumber \\
	&\quad \le 2\EE_\init\Bigl[\big\| W^\top \phi_{W_1}(s, a) - W^\top \phi_0(s, a) \big\|_{2, \mu}^2\Bigr] + 2\EE_\init\Bigl[\big\| W^\top \phi_{W_2}(s, a) - W^\top \phi_0(s, a) \big\|_{2, \mu}^2\Bigr] \nonumber \\
	&\quad = \cO(B^3\cdot m^{-1/2}).
	\end{align*}
	Moreover, following from the Cauchy-Schwartz inequality, we have that $\norm{\cdot}_{1, \mu} \le \norm{\cdot}_{2, \mu}$. Thus, by Jensen's inequality, we have that
	\begin{align*}
	&\EE_\init\Bigl[\big\| W^\top \phi_{W_1}(s, a) - W^\top \phi_{W_2}(s, a) \big\|_{1, \mu}\Bigr] \nonumber\\
	&\quad \le \EE_\init \Bigl[ \norm[\big]{W^\top \phi_{W_1}(s, a) - W^\top \phi_{W_2}(s, a)  }_{2, \mu} \Bigr] \nonumber\\
	&\quad = \cO(B^{3/2}\cdot m^{-1/4}),
	\end{align*}
    which completes the proof of Lemma \ref{lem:linear_nn}.
\end{proof}

\subsection{Proof of Lemma \ref{lem:bound_nn}}
\label{sec:pf_lem_bound_nn}
In what follows, we present the proof of Lemma \ref{lem:bound_nn}.
\begin{proof}
	Recall that we write $u_W(s, a) = W^\top \phi_W(s, a)$ and $u_0(s, a)= u_{W_0}(s, a)$. Then, we have
	\begin{align}
	\label{eq:b_nn1}
	\big|W^\top \phi_{W'}(s, a)\big| &\le \bigl|u_0(s, a) \bigr| + \bigl|(W - W')^\top \phi_{W'}(s, a)\bigr| + \bigl| u_{W'}(s, a) - u_0(s, a)\bigr| \nonumber \\
	&\le \bigl|u_0(s, a) \bigr| + \norm{W - W'}_2\cdot \norm[\big]{\phi_{W'}(s, a)}_2 + \bigl| u_{W'}(s, a) - u_0(s, a)\bigr|,
	\end{align}
	where the last inequality follows from the Cauchy-Schwartz inequality. It suffices to upper bound the three terms on the right-hand side of \eqref{eq:b_nn1}. Note that we have $W, W'\in S_B$ and $\norm{\phi_{W'}(s, a)}_2 \le 1$. We have that
	\begin{align}
	\label{eq:b_nn2}
	\norm{W - W'}_2 \cdot \norm[\big]{\phi_{W'}(s, a)}_2 \le 2B.
	\end{align}
	It remains to upper bound the term $ | u_{W'}(s, a) - u_0(s, a)|$ in \eqref{eq:b_nn1}. Note that $u_W(s, a)$ is almost everywhere differentiable with respect to $W$. Also, it holds that $\nabla_W u_W(s, a) = \phi_W(s, a)$. Thus, following from the mean-value theorem and the Cauchy-Schwartz inequality, we have that
	\begin{align}
	\label{eq:b_nn3}
	\bigl|u_{W'}(s, a) - u_0(s, a)\bigr| \le \sup_{W\in S_B} \norm[\big]{\phi_W(s, a)}_2 \cdot \norm{W' - W_0}_2 \le B,
	\end{align}
	where the second inequality follows from the fact that $\|\phi_W(s, a)\|_2 \le 1$ and $W' \in S_B$.
	Plugging \eqref{eq:b_nn2} and \eqref{eq:b_nn3} into \eqref{eq:b_nn1}, we have that
	\begin{align*}
	\sup_{(s,a) \in \cS\times \cA}\bigl| W^\top \phi_{W'}(s, a) \bigr| \le \sup_{(s,a) \in \cS\times \cA}\bigl|u_0(s, a) \bigr| + 3B.	
	\end{align*}
	Following from Assumption \ref{asp:nn}, we have that $\sup_{(s,a) \in \cS\times \cA} | W^\top \phi_{W'}(s, a) |$ is sub-Gaussian. Furthermore, it holds that
	\begin{align*}
	\EE_\init \Bigl[\sup_{(s,a) \in \cS\times \cA}\bigl| W^\top \phi_{W'}(s, a) \bigr|^2 \Bigr] \le 2\EE_\init \Bigl[\sup_{(s,a) \in \cS\times \cA}\bigl| u_0(s, a) \bigr|^2 \Bigr] + 18B^2 \le 2M_0^2 + 18B^2
	\end{align*}
	and that
	\begin{align*}
	\PP\Bigl( \sup_{(s,a) \in \cS\times \cA}\bigl| W^\top \phi_{W'}(s, a) \bigr| > t \Bigr) &\le \PP\Bigl (\sup_{(s, a) \in \cS\times \cA} \bigl| u_0(s, a) \bigr| + 3B > t \Bigr) \\
	&\le \exp\bigl(-v\cdot (t-3B)^2\bigr) \le \exp(-v\cdot t^2 / 2)
	\end{align*}
	for $t > 2M_0 + 6B$. Thus, we complete the proof of Lemma \ref{lem:bound_nn}.
\end{proof}



\section{Neural Temporal Difference}
\label{sec:td_learning}
In this section, we introduce neural TD \citep{cai2019neural}, which computes $\omega_k$ in Algorithm \ref{alg:ac_gail}. Specifically, neural TD solves the optimization problem in \eqref{eq:def_tdlearn} via the update in \eqref{eq:td_update}, which is presented in Algorithm \ref{alg:td_learn}.

\begin{algorithm}
	\caption{Neural TD}
	\label{alg:td_learn}
	\begin{algorithmic}[1]
		\REQUIRE Policy $\pi$, reward function $r$, initialization $W_0, b$, number of iterations $T_{\mathrm{TD}}$ of neural TD, and stepsize $\alpha$ of neural TD.
		\STATE {\bf Initialization.} Set $S_{B_\omega} \leftarrow \{W\in\RR^{md}\given \norm{W-W_0}_2 \le B_\omega\}$ and $\omega(0) \leftarrow W_0$.
		\FOR{$ j= 0, \ldots, T_{\td}-1 $}
		\STATE Sample $(s, a, s', a')$, where $(s, a) \sim \rho_{\pi}$, $s' \sim P(\cdot \given s, a)$, and $ a'\sim \pi(\cdot \given s')$.
		\STATE Compute the Bellman residual $\delta(j) = Q_{\omega(j)}(s, a) - (1-\gamma) \cdot r(s, a) - \gamma \cdot Q_{\omega(j)}(s', a')$.
		\STATE Update $\omega$ via $\omega(j+1) \leftarrow \proj_{S_{B_\omega}}\bigl\{ \omega(j) - \eta\cdot \delta(j)  \cdot \phi_{\omega(j)}(s,a)\bigr\}$.
		\ENDFOR
		\ENSURE Output $\bar \omega = T^{-1} \sum_{t=0}^{T_{\td}-1} \omega(j)$.
	\end{algorithmic}
\end{algorithm}

\subsection{Proof of Proposition \ref{prop:conv_td}}
\label{sec:pf_conv_td}
\begin{proof}
	We obtain the following proposition from \cite{cai2019neural}, which characterizes the convergence of Algorithm \ref{alg:td_learn}.
	\begin{proposition}[Proposition 4.7 in \cite{cai2019neural}]
		\label{th:cai_td_conv}
		We set $\alpha = \min\{ (1-\gamma) / 8, T^{-1/2}_{\td} \}$ in  Algorithm \ref{alg:td_learn}. Let $Q_{\bar \omega}(s, a)$ be the state-action value function associated with the output $\bar\omega$. Under Assumption \ref{asp:linear_nn}, it holds for any policy $\pi$ and reward function $r(s, a)$ that
		\begin{align}
		\label{eq:td-conv}
		\EE_\init\Bigl[ \norm[\big]{Q_{\bar \omega} (s, a) - Q^\pi_r(s, a) }^2_{2,\rho_\pi} \Bigr] &= 2\EE_\init \Bigl[ \norm[\big]{ \proj_{\cF_{B_\omega, m}} Q^\pi_r(s, a) - Q^\pi_r(s, a) }_{2, \rho_\pi}^2
		\Bigr] \nonumber\\
		&\quad + \cO(B_\omega^2 \cdot T^{-1/2}_{\td} + B_\omega^3 \cdot m^{-1/2} + B_\omega^{5/2} \cdot m^{-1/4} ),
		\end{align}
		where $\cF_{B_\omega, m}$ is defined in Definition \ref{def:func_class}.
	\end{proposition}
	Recall that we denote by $\phi_0(s, a)$ the feature vector corresponding to the random initialization in \eqref{eq:init}. We establish the following lemma to upper bound the bias $ \EE_\init [ \norm{ \proj_{\cF_{B_\omega, m}} Q^\pi_r(s, a) - Q^\pi_r(s, a) }^2_{2, \rho_\pi}
	] $ in \eqref{eq:td-conv} of Proposition \ref{th:cai_td_conv} when the reward function $r(s, a)$ belongs to the reward function class $\cR_\beta$.
	\begin{lemma}
		\label{lem:func_trans}

		 We consider any reward function $r_\beta(s, a) \in \cR_\beta$ and policy $\pi$. Under Assumptions \ref{asp:measures} and \ref{asp:nn}, it holds for $B_\omega > B_\beta + (1-\gamma)^{-1}\cdot \gamma\cdot B_P\cdot (2M_0 + 3B_\beta) $ and an absolute constant $C_v = (2 \cdot \gamma^2 \cdot B_P^2)^{-1} \cdot(1-\gamma)^2 \cdot  v$ that
		\begin{align*}
		\EE_{\init} \Bigl[ \norm[\big]{ \proj_{\cF_{B_\omega, m}} Q_{r_\beta}^\pi (s, a) - Q_{r_\beta}^\pi(s, a)  }^2_{2, \rho_\pi} \Bigr] = \cO\bigl(B_\beta^3\cdot m^{-1/2}+ B_\omega^2 \cdot m^{-1} + B_\omega^2 \cdot \exp(-C_v\cdot B_\omega^2)\bigr).
		\end{align*}
	\end{lemma}
	\begin{proof}
		See \S\ref{sec:pf_lem_func_trans} for a detailed proof.
	\end{proof}
	Combining Proposition \ref{th:cai_td_conv} and Lemma \ref{lem:func_trans}, for $B_\omega > B_\beta + (1-\gamma)^{-1}\cdot \gamma\cdot B_P\cdot (2M_0 + 3B_\beta) $, we have for any $\pi$ that
	\begin{align*}
	\EE_{\init}\Bigl[ \norm[\big]{Q_{\bar \omega} (s, a) - Q^\pi_{r_\beta}(s, a) }^2_{2,\rho_\pi} \Bigr] = \cO\bigl(B_\omega^2 \cdot T^{-1/2}_{\td} + B_\omega^3 \cdot m^{-1/2} + B_\omega^{5/2} \cdot m^{-1/4} + B_\omega^2\cdot \exp(-C_v\cdot B_\omega^2)\bigr).
	\end{align*}
	Finally, by setting $T_{\td} = \Omega(m)$, we have that 
	\begin{align*}
	\EE_{\init}\Bigl[ \norm[\big]{Q_{\bar \omega} (s, a) - Q^\pi_{r_\beta}(s, a) }^2_{2,\rho_\pi} \Bigr] = \cO\bigl(B_\omega^3 \cdot m^{-1/2} + B_\omega^{5/2}\cdot m^{-1/4} + B_\omega^2\cdot \exp(-C_v\cdot B_\omega^2)\bigr),
	\end{align*} 
	which completes the proof of Proposition \ref{prop:conv_td}.
\end{proof}

\subsection{Proof of Lemma \ref{lem:func_trans}}

\label{sec:pf_lem_func_trans}
\begin{proof}
	For notational simplicity, we write $\vartheta(s, a; w) = \ind{\{| w^\top (s, a)| >0 \}}\cdot (s, a)$. Under Assumption \ref{asp:nn}, we have that
	\begin{align}
	\label{eq:ft0}
	P(s'\given s, a) = \int \vartheta(s, a; w)^\top \varphi(s'; w) \rd q(w),\quad \text{where  } \sup_{ w} \biggl\| \int\varphi(s; w) \rd s \biggr\|_2 \le B_P.
	\end{align}
	Thus, since $r_\beta = (1-\gamma)^{-1}\cdot u_\beta(s, a)$, we have that
	\begin{align*}
	Q^\pi_{r_\beta}(s, a) &= (1-\gamma)\cdot r_\beta(s, a) + \gamma \cdot \int _\cS P(s'\given s, a)\cdot  V^\pi_{r_\beta}(s') \rd s'  \nonumber \\
	& = u_\beta(s, a) + \int_\cS \gamma \cdot V^\pi_{r_\beta}(s') \cdot \int\vartheta(s, a; w)^\top \varphi(s'; w) \rd q(w) \rd s' \nonumber \\
	& = u_\beta(s, a) + \int \vartheta(s, a; w)^\top \biggl( \gamma\cdot \int_\cS \varphi(s'; w)V^\pi_{r_\beta}(s') \rd s' \biggr) \rd q(w),
	\end{align*}
	where the second equality follows from \eqref{eq:ft0} and the last equality follows from Fubini's theorem.
	In the sequel, we define 
	\begin{align}
	\label{eq:def-alpha}
	\alpha(w) = \gamma\cdot \int_\cS \varphi(s'; w)V^\pi_{r_\beta}(s') \rd s'. 
	\end{align}
	Note that $\alpha(w) \in \RR^d$.
	Then, we have that 
	\begin{align*}
	Q^\pi_{r_\beta}(s, a) = u_\beta(s, a) + \int \vartheta(s, a; w)^\top \alpha(w) \rd q(w).
	\end{align*}
	To prove Lemma \ref{lem:func_trans}, we first approximate $Q^{\pi}_{r_\beta}(s, a)$ by 
	\begin{align}
	\label{eq:ft-2}
	\bar Q(s, a) = u_\beta(s, a) + \int \vartheta(s, a; w)^\top \bar \alpha(w) \rd q(w),
	\end{align}
	 where $\bar \alpha (w) = \alpha(w) \cdot \ind\{\norm{\alpha(w)}_2 \le K\}$ for an absolute constant $K>0$ specified later. Then, it holds for any $(s, a)\in \cS\times \cA$ that
	\begin{align*}
	 \bigl|\bar Q(s, a) - Q^\pi_{r_\beta}(s, a)\bigr| &\le  \int \Bigl| \vartheta(s, a; w)^\top \bigl(\bar \alpha(w) - \alpha(w)\bigr)\Bigr| \rd q(w)   \nonumber \\
	& \le  \int \norm[\big]{\vartheta(s, a; w)}_2 \cdot \norm[\big]{\bar \alpha(w) - \alpha(w)}_2\rd q(w) \nonumber \\
	& \le\sup_w \norm[\big]{\bar \alpha(w) - \alpha(w)}_2,
	\end{align*}
	where the second inequality follows from the Cauchy-Schwartz inequality and the last inequality follows from the fact that $\norm{\vartheta(s, a; w)}_2 \le 1$. Thus, we have that
	\begin{align}
	\label{eq:ft2}
	\norm[\big]{\bar Q(s, a) - Q^\pi_{r_\beta}(s, a)}_{2, \rho_\pi} \le \norm[\big]{\bar Q(s, a) - Q^\pi_{r_\beta}(s, a)}_\infty \le \sup_w \norm[\big]{\bar \alpha(w) - \alpha(w)}_2. 
	\end{align}
	We now upper bound the right-hand side of \eqref{eq:ft2}. To this end, we show that $\sup_w\|\alpha(w)\|_2$ is sub-Gaussian in the sequel. By the definition of $\alpha(w)$ in \eqref{eq:def-alpha}, we have that
	\begin{align}
	\label{eq:ft4}
	\sup_w \bigl\| \alpha(w) \bigr\|_2 &=\gamma\cdot \biggl\|\int_\cS \varphi(s'; w)V^\pi_{r_\beta}(s') \rd s' \biggr\|_2 \nonumber \\
	& \le \gamma \cdot \sup_{s'\in \cS} \bigl| V^\pi_{r_\beta}(s')  \bigr|\cdot \sup_w \biggl\|\int_\cS \varphi(s'; w)\rd s' \biggr\|_2 \nonumber\\
	& \le \gamma \cdot B_P \cdot \sup_{s'\in \cS} \bigl|V^\pi_{r_\beta}(s') \bigr| \nonumber \\
	&\le \gamma  \cdot(1-\gamma)^{-1}\cdot B_P\cdot  \sup_{(s, a)\in\cS\times \cA} \bigl|u_\beta(s, a) \bigr|,
	\end{align}
	where the second inequality follows from Assumption \ref{asp:nn} and the third inequality follows from the fact that $V_{r_\beta}^\pi(s) = \EE_{(s', a')\sim\nu_\pi(s)} [r_\beta(s', a') ] $. Here we denote by $\nu_\pi(s)$ the state-action visitation measure starting from the state $s$ and following the policy $\pi$. Following from Lemma \ref{lem:bound_nn}, we have that $ \sup_w \| \alpha(w) \|_2 $ is sub-Gaussian. By Lemma \ref{lem:bound_nn} and \eqref{eq:ft4}, it holds for $t > (1-\gamma)^{-1}\cdot \gamma\cdot B_P\cdot (2M_0 + 3B_\beta) $ that
	\begin{align}
	\label{eq:ft5}
	\PP\Bigl(\sup_w \bigl\| \alpha(w) \bigr\|_2 > t \Bigr) &\le \PP\Bigl( \gamma  \cdot(1-\gamma)^{-1}\cdot B_P\cdot  \sup_{(s, a)\in\cS\times \cA} \bigl|u_\beta(s, a) \bigr| > t \Bigr) \nonumber \\
	&\le \exp\biggl( -\frac{v\cdot(1-\gamma)^2 \cdot t^2}{2 \gamma^2 \cdot B_P^2 } \biggr).
	\end{align}
	Let the absolute constant $K$ satisfy that $K > (1-\gamma)^{-1}\cdot \gamma\cdot B_P\cdot (2M_0 + 3B_\beta)$ in \eqref{eq:ft5}. For notational simplicity, we write $C_v = (2 \cdot \gamma^2 \cdot B_P^2)^{-1} \cdot  v\cdot(1-\gamma)^2$.
	By the fact that $\|\bar \alpha(w) - \alpha(w)\|_2 = \norm{\alpha(w)}_2 \cdot  \ind\{ \norm{\alpha(w)}_2 > K \}$, we have that
	\begin{align*}
	\sup_w \norm[\big]{\bar \alpha(w) - \alpha(w)}_2 \le \sup_w \norm[\big]{\alpha(w)}_2 \cdot \ind\Bigl\{ \sup_w \norm{\alpha(w)}_2 > K \Bigr\}.
	\end{align*}
	Following from \eqref{eq:ft2} and \eqref{eq:ft5}, we have that
	\begin{align}
	\label{eq:ft6}
	&\EE_\init\Bigl[\norm[\big]{\bar Q(s, a) - Q^\pi_{r_\beta}(s, a)}_{2, \rho_\pi}\Bigr] \nonumber \\
	&\quad  \le \EE\biggl[ \sup_w \norm[\big]{\alpha(w)}_2 \cdot \ind\Bigl\{ \sup_w \norm{\alpha(w)}_2 > K \Bigr\} \biggr] \nonumber \\
	&\quad  \le \int_{0}^{K} t \cdot \PP\Bigl(\sup_w \norm{ \alpha(w) }_2 > K\Bigr) \rd t + \int_{K}^\infty t\cdot \PP\Bigl(\sup_w \norm{ \alpha(w) }_2 > t \Bigr) \rd t\nonumber \\
	& \quad = \cO\bigl( K^2 \cdot \exp( - C_v\cdot K^2) \bigr).
	\end{align}
	We now construct $\hat Q(s, a)\in\cF_{K, m}$, which approximates $\bar Q(s, a)$ defined in \eqref{eq:ft-2}. We define 
	\begin{align*}
	f(s, a) = \int \vartheta(s, a; w)^\top \bar \alpha(w)\rd q(\omega) .
	\end{align*}
	 Then, we have that $\bar Q(s, a) = u_\beta(s, a) + f(s, a)$.
	Note that $f(s, a)$ belongs to the following function class,
	\begin{align*}
	\tilde \cF_{K, \infty} = \biggl\{ \int \vartheta(s, a; w)^\top \alpha(w)\rd q(\omega) \,\bigg|\, \sup_w \norm[\big]{\alpha(w)}_2 \le K \biggr\}.
	\end{align*}
	We now show that $f(s, a)$ is well approximated by the following function class,
	\begin{align*}
	\tilde \cF_{K, m} = \biggl\{ W^\top \phi_0(s, a) = \frac{1}{\sqrt{m}} \sum_{l=1}^{m} [W]_l^\top \vartheta\bigl(s, a; [W]_l\bigr) \,\bigg|\, \sup_l\norm[\big]{[W]_l}_2 \le K / \sqrt{m} \biggr\},
	\end{align*}
	where $\phi_0(s, a)$ is the feature vector corresponding to the random initialization. We obtain the following lemma from \cite{NIPS2008_3495}, which characterizes the approximation error of $\tilde \cF_{K, \infty}$ by $\tilde \cF_{K, m}$.
	\begin{lemma}[Lemma 1 in \cite{NIPS2008_3495}]
		\label{lem:proj_err0}
		For any $f(s, a) \in \tilde \cF_{K, \infty}$, it holds with probability at least $1-\delta$ that
		\begin{align*}
		\norm[\big]{ \proj_{\tilde \cF_{K, m}} f(s, a) - f(s, a)}_{2, \mu} \le K\cdot m^{-1/2} \cdot \bigl(1 + \sqrt{2\log(1/\delta)}\bigr),
		\end{align*}
		where $\mu\in\cP(\cS\times \cA)$.
	\end{lemma}
	Lemma \ref{lem:proj_err0} implies that there exists $\hat f(s, a) \in \tilde \cF_{K, m}$ such that
	\begin{align}
	\EE_\init \Bigl[\norm[\big]{ \hat f(s, a) - f(s, a)}^2_{2, \rho_\pi}\Bigr]   &= \int_0^\infty \PP\Bigl( \norm[\big]{ \hat  f(s, a) - f(s, a)}^2_{2, \rho_\pi}>y \Bigr) \rd y \nonumber \\
	& \le \int_0^\infty y \cdot \exp\bigl(-1/2\cdot (\sqrt{my} / K - 1)^2\bigr) = \cO(K^2 / m).
	\label{eq:proj_err0}
	\end{align}
	By the fact that $\hat f(s, a) \in \tilde \cF_{K, m}$ and the definition of $\cF_{K, m}$ in Definition \ref{def:func_class}, we have that $\hat f(s, a) \in\cF_{K, m}- u_0(s, a) $. Let 
	\begin{align*}
	\hat Q(s, a) = \beta^\top \phi_0(s, a) +\hat f(s, a) = (\beta + W_f)^\top \phi_0(s, a) .
	\end{align*} 
	We then have that $\hat Q(s, a) \in \cF_{B_\beta+ K, m}$ and that
	\begin{align} 
	\label{eq:ft8}
	\EE_\init \Bigl[\norm[\big]{ \bar Q(s, a) - \hat Q(s, a) }^2_{2, \rho_k}\Bigr] &\le 2\EE_\init \Bigl[\norm[\big]{ u_\beta(s, a) - \beta^\top\phi_0(s, a)}^2_{2, \rho_\pi}\Bigr] +2\EE_\init\Bigl[\norm[\big]{ \hat f(s, a) - f(s, a)}^2_{2, \rho_\pi}\Bigr]\nonumber \\ &= \cO(B_\beta^3\cdot m^{-1/2}+ K^2 \cdot m^{-1}),
	\end{align}
	where the last inequality follows from Assumption \ref{asp:linear_nn}, Lemma \ref{lem:linear_nn}, and \eqref{eq:proj_err0}.
	
	Finally, we set $B_\omega = K + B_\beta > B_\beta + (1-\gamma)^{-1}\cdot \gamma\cdot B_P\cdot (2M_0 + 3B_\beta)$. Combining \eqref{eq:ft6} and \eqref{eq:ft8}, we have that
	\begin{align*}
	\EE_\init \Bigl[\norm[\big]{ Q_{r_\beta}^\pi(s, a) - \hat Q(s, a) }^2_{2, \rho_k}\Bigr] &\le 2\EE_\init \Bigl[\norm[\big]{ \bar Q(s, a) - \hat Q(s, a) }^2_{2, \rho_k}\Bigr]+2\EE_\init \Bigl[\norm[\big]{ \bar Q(s, a) - Q_{r_\beta}^\pi(s, a) }^2_{2, \rho_k}\Bigr] \nonumber \\
	&= \cO\bigl(B_\beta^3\cdot m^{-1/2}+ B_\omega^2 \cdot m^{-1} + B_\omega^2 \cdot \exp(-C_v\cdot B_\omega^2)\bigr),
	\end{align*}
	where $\hat Q(s, a) \in \cF_{B_\omega, m}$. Thus, we complete the proof of Lemma \ref{lem:func_trans}.
\end{proof}

\section{Proofs of Auxiliary Results}
\label{sec:pf_aux}
In what follows, we present the proofs of the lemmas in \S\ref{sec:alg}-\ref{sec:analysis}.
\subsection{Proof of Proposition \ref{prop:npg_form}}
\label{sec:pf_prop_npg_form}
\begin{proof}
	By the definition of the neural network in \eqref{eq:def_nn}, we have for any $(s, a) \in \cS\times \cA$ that
	$\nabla_W u_W(s, a) = \phi_W(s, a)$ almost everywhere.
	We first calculate $\nabla_\theta L(\theta, \beta)$. Following from the policy gradient theorem \citep{sutton2018reinforcement} and the definition of $L(\theta, \beta)$ in \eqref{eq:def_il}, we have that
	\begin{align}
	\label{eq:nf3}
	\nabla_\theta L(\theta, \beta) &= -\nabla_\theta J(\pi_\theta; r_\beta) \nonumber \\
	&= - \EE_{ \nu_{\pi_\theta}}\bigl[ Q^{\pi_\theta}_{r_\beta}(s, a) \cdot \nabla_\theta \log \pi_\theta(a\given s) \bigr].
	\end{align}
    Following from the parameterization of $\pi_\theta$ in \eqref{eq:para_pi} and the definition of $\phic_\theta(s, a)$ in \eqref{eq:def_phic} of Proposition \ref{prop:npg_form}, we have that
	\begin{align}
	\label{eq:nf_2}
	\nabla_\theta \log \pi_\theta (a\given s) &= \tau \cdot \phi_\theta(s, a) - \frac{ \sum_{a'\in \cA} \tau \cdot \exp \bigl(\tau \cdot \theta^\top \phi_\theta(s, a')\bigr) \cdot \phi_\theta(s, a') }{\sum_{a'\in \cA} \exp \bigl(\tau \cdot \theta^\top \phi_\theta(s, a')\bigr)} \nonumber \\
	& = \tau \cdot \Bigl(\phi_\theta(s, a) - \tau \cdot \EE_{a'\sim\pi_\theta(\cdot \given s)}\bigl[ \phi_\theta(s, a') \bigr] \Bigr) = \tau \cdot \phic_\theta(s, a).
 	\end{align}
 	Plugging \eqref{eq:nf_2} into \eqref{eq:nf3}, we have that
 	\begin{align*}
 	\nabla_\theta L(\theta, \beta) = - \tau \cdot \EE_{ \nu_{\pi_\theta}}\bigl[ Q^{\pi_\theta}_{r_\beta}(s, a)\cdot \phic_\theta(s, a) \bigr].
 	\end{align*}
    It remains to calculate $\cI(\theta)$ and $\nabla_\beta L(\theta, \beta)$. By \eqref{eq:nf_2} and the definition of $\cI(\theta)$ in \eqref{eq:fisher}, it holds that
 	\begin{align*}
 	\cI(\theta) &= \EE_{ \nu_{\pi_\theta}}\bigl[\nabla\log\pi_\theta(a\given s) \nabla\log\pi_\theta(a\given s)^\top \bigr] \nonumber \\
 	& = \tau^2 \cdot \EE_{ \nu_{\pi_\theta}}\bigl[ \phic_\theta(s, a)\phic_\theta(s, a)^\top \bigr]. 
 	\end{align*}
 	By the definition of the objective function $L(\theta, \beta)$ in \eqref{eq:def_il}, it holds that
 	\begin{align*}
 	\nabla_\beta L(\theta, \beta) & = \nabla_\beta J(\pi_\rE; r_\beta) - \nabla_\beta J(\pi_\theta; r_\beta) - \lambda \cdot \nabla_\beta \psi(\beta) \nonumber \\
 	& = \EE_{\nu_\rE}\bigl[\nabla_\beta r_\beta(s, a)\bigr] - \EE_{\nu_{\pi_\theta}}\bigl[\nabla_\beta r_\beta(s, a)\bigr] - \lambda \cdot \nabla_\beta \psi(\beta)\nonumber\\
 	&= (1-\gamma)^{-1}\cdot \EE_{\nu_\rE}\bigl[\phi_\beta(s, a)\bigr] - (1-\gamma)^{-1} \cdot  \EE_{\nu_{\pi_\theta}}\bigl[\phi_\beta(s, a)\bigr] - \lambda \cdot \nabla_\beta \psi(\beta).
 	\end{align*}
 	Thus, we complete the proof of Proposition \ref{prop:npg_form}.
\end{proof}

\subsection{Proof of Lemma \ref{lem:pe0}}
\label{sec:pf_lem_pe0}

\begin{proof}
	The proof of Lemma \ref{lem:pe0} is similar to that of Lemmas 5.4 and 5.5 in \cite{wang2019neural}. 
	By direct calculation, we have that
	\begin{align*}
	&\eta \cdot \EE_{d_\rE}\Bigl[ \bigl\langle Q^{\pi_k}_{r_k}(s, \cdot), \pi_\rE^s - {\pi_k^s} \bigr\rangle_\cA \Bigr] = \kl^{d_\rE}(\pi_\rE\,\|\, \pi_k ) - \kl^{d_\rE}(\pi_\rE\,\|\, \pi_{k+1} ) + \eta \cdot \Delta_k^{\text{(i)}},
	\end{align*}
	where $\Delta_k^{\text{(i)}}$ takes the form of
	\begin{align}
	\label{eq:policy_error}
	\Delta_k^{\text{(i)}} &=\eta^{-1}\cdot\biggl\{ \EE_{d_\rE}\Bigl[ \bigl\langle \log(\pi_{k+1}^s / \pi _k ^s) - \eta\cdot Q ^{\pi _k} _{r_k}(s, \cdot), \pi_\rE^s - \pi_k^s \bigr\rangle_\cA + \bigl\langle \log({\pi_{k+1}^s} / {\pi_{k}^s}), \pi_k^s - \pi_{k+1}^s \bigr\rangle_\cA \Bigr]   - \kl^{d_\rE}(\pi_{k+1}^s \,\|\, \pi_{k}^s )\biggr\} \nonumber \\
	&=  \underbrace{\eta^{-1}\cdot \EE_{d_\rE}\Bigl[\bigl\langle \log(\pi _ {k+1} ^ s / \pi_k^s) - \eta \cdot  \hat Q_{\omega_k}(s, \cdot), \pi_\rE^s - \pi_{k}^s \bigr\rangle_\cA\Bigr]}_{\displaystyle\text{(i.a)}} + \underbrace{\EE_{d_\rE}\Bigl[ \bigl\langle  \hat Q_{\omega_k}(s, \cdot) - Q_{r_k}^{\pi_k}(s, \cdot), \pi_\rE^s - \pi_k^s \bigr\rangle_\cA\Bigr]}_{\displaystyle\text{(i.b)}}  \nonumber \\
	&\quad  + \underbrace{\eta^{-1}\cdot \EE_{d_\rE}\Bigl[\bigl\langle \log({\pi_{k+1}^s} / {\pi_{k}^s}), \pi_k^s - \pi_{k+1}^s \bigr\rangle_\cA - \kl(\pi_{k+1}^s \,\|\, \pi_{k}^s )\Bigr]}_{\displaystyle\text{(i.c)}}.
	\end{align}
	The following lemmas upper bound $ \Delta_k^{\text{(i)}} $
	by upper bounding terms (i.a), (i.b), and (i.c) on the right-hand side of \eqref{eq:policy_error}, respectively. Note that the expectation $\EE_{\init, d_\rE}$ is taken with respect to the random initialization in \eqref{eq:init} and $s \sim d_\rE$.
	
	\begin{lemma}[Upper Bound of Term (i.a) in \eqref{eq:policy_error}]
		\label{lem:bound_pe1}
		Under Assumptions \ref{asp:linear_nn} and \ref{asp:more}, we have that
		\begin{align*}
		&\EE_{\init, d_\rE}\biggl[\Bigl| \inp[\big]{\log (\pi_ {k+1}^s / \pi _{k}^s) - \eta \cdot \hat Q_{\omega_k}(s, \cdot) }{ \pi_\rE^s - \pi_k^s }_\cA \Bigr|\biggr] \nonumber \\
		&\quad = \eta\cdot 2\sqrt{2}\cdot C_h \cdot B_\theta^{1/2} \cdot \sigma^{1/2} \cdot N^{-1/4} + \cO(\tau_{k+1} \cdot B_\theta^{3/2} \cdot m^{-1/4} + \eta \cdot B_\theta^{5/4} \cdot m^{-1/8} ),
		\end{align*}
		where $C_h$ is defined in Assumption \ref{asp:visitation} and $\sigma$ is defined in Assumption \ref{asp:more}.
	\end{lemma}
\begin{proof}
	See \S\ref{sec:pf_lem_bound_pe1} for a detailed proof.
\end{proof}
	
\begin{lemma}[Upper Bound of Term (i.b) in \eqref{eq:policy_error}]
	\label{lem:bound_pe2}
	Under Assumption \ref{asp:visitation}, we have that
	\begin{align*}
	\EE_{\init, d_\rE}\Bigl[ \bigl\langle  \hat Q_{\omega_k}(s, \cdot) - Q_{r_k}^{\pi_k}(s, \cdot), \pi_\rE^s - \pi_k^s \bigr\rangle_\cA \Bigr] \le C_h \cdot \err_{Q, k},
	\end{align*}
	where $\err_{Q, k}$ takes the form of 
	\begin{align}
	\label{eq:def_errq}
	\err_{Q, k} = \EE_\init\Bigl[ \norm[\big]{Q_{r_k}^{\pi_k}(s, a) - \hat Q_{\omega_k}(s, a) }_{2,  \rho_k} \Bigr].
	\end{align}
\end{lemma}
\begin{proof}
	See \S\ref{sec:pf_lem_bound_pe2} for a detailed proof.
\end{proof}

\begin{lemma}[Upper Bound of Term (i.c) in \eqref{eq:policy_error}]
	\label{lem:bound_pe3}
	Under Assumptions \ref{asp:linear_nn} and \ref{asp:bound_init}, we have that
	\begin{align*}
	&\EE_{\init, d_\rE}\biggl[\Bigl|\inp[\big]{\log (\pi_{k+1}^s / \pi_k^s)}{\pi_k^s - \pi_{k+1}^s}_\cA \Bigr| -\kl(\pi_ {k+1}^s \,\|\, \pi_k^s) \biggr]  \nonumber\\ 
	&\quad = \eta^2 \cdot (M_0^2 + 9B_\theta^2) + \cO(\tau_{k+1}\cdot B_\theta^{3/2}\cdot m^{-1/4}),
	\end{align*}
	where $M_0$ is defined in Assumption \ref{asp:bound_init}.
\end{lemma}
\begin{proof}
	See \S\ref{sec:pf_lem_bound_pe3} for a detailed proof.
\end{proof}
Finally, by Lemmas \ref{lem:bound_pe1}-\ref{lem:bound_pe3}, under Assumptions \ref{asp:nn} and \ref{asp:more}, we obtain from \eqref{eq:policy_error} that
\begin{align*}
\EE_{\init}\bigl[|\Delta_k^{\text{(i)}}|\bigr]   &= 2\sqrt{2} C_{h} \cdot B_{\theta}^{1/2} \cdot\sigma^{1/2}\cdot N^{-1/4} + C_{h}\cdot \err_{Q, k} + \eta\cdot (M_0^2 + 9 B_\theta^2) \nonumber \\
&\quad +  \cO(\eta^{-1}\cdot \tau_{k+1} \cdot B_\theta^{3/2} \cdot m^{-1/4} + B_\theta^{5/4} \cdot m^{-1/8} ). \nonumber
\end{align*}
Here $M_0$ is defined in Assumption \ref{asp:bound_init}, $\tau_{k+1}$ is the inverse temperature parameter of $\pi_{k+1}$ defined in \eqref{eq:para_pi}, $\sigma$ is defined in Assumption \ref{asp:more}, and $\err_{Q, k}$ is defined in \eqref{eq:def_errq} of Lemma \ref{lem:bound_pe2}. Following from Proposition \ref{prop:conv_td}, we have that 
\begin{align*}
C_{h}\cdot \epsilon_{Q,k} = \cO\bigl(B_\omega^3 \cdot m^{-1/2} + B_\omega^{5/2}\cdot m^{-1/4} + B_\omega^2\cdot \exp(-C_v\cdot B_\omega^2)\bigr).  
\end{align*}
Thus, we complete the proof of Lemma \ref{lem:pe0}.
\end{proof}

\subsection{Proof of Lemma \ref{lem:reward}}
\label{sec:pf_lem_reward}
\begin{proof}
	We consider a fixed $\beta' \in S_{B_\beta}$. For notational simplicity, we write $r' = r_{\beta'}(s, a)$, $r_{k} = r_{k}(s,a)$ and $\phi_{\beta} = \phi_{\beta}(s, a)$. By the parameterization of $r_\beta(s, a)$ in \eqref{eq:para_r}, we have that 
	\begin{align}
	\label{eq:reward_diff}
	&L(\theta_k, \beta') - L(\theta_k, \beta_k) = \langle r' - r_k, \nu_\rE - \nu_k \rangle_{\cS\times \cA} + \lambda\cdot \psi(\beta_k) - \lambda\cdot \psi(\beta') \nonumber\\
	& \quad = (1-\gamma)^{-1}\cdot \Bigl(\bigl\langle\phi_{\beta_k}^\top ( \beta' - \beta_k ), \nu_\rE -  \nu_k \bigr\rangle_{\cS\times \cA} + \langle \phi_{\beta'}^\top \beta' - \phi_{\beta_k}^\top \beta', \nu_\rE -  \nu_k \rangle_{\cS\times \cA} 
	\Bigr)   + \lambda \cdot \bigl(\psi(\beta) - \psi(\beta') \bigr) \nonumber \\
	&\quad \le  (\beta' - \beta_k)^\top \nabla_\beta L(\theta_k, \beta_k) + (1-\gamma)^{-1}\cdot \bigl(\|\phi_{\beta_k}^\top \beta' - \phi_{\beta'}^\top \beta' \|_{1, \nu_k} + \|\phi_{\beta_k}^\top \beta' - \phi_{\beta'}^\top \beta' \|_{1, \nu_\rE}\bigr),
	\end{align}
	where the last inequality follows from \eqref{eq:pg_form} of Proposition \ref{prop:npg_form}.
	Then, we have that
	\begin{align*}
	&\EE_\init\bigl[L(\theta_k, \beta') - L(\theta_k, \beta_k)\bigr] \nonumber\\
	&\quad \le \EE_\init \Bigl[ (\beta' - \beta_k)^\top \nabla_\beta L(\theta_k, \beta_k) + (1-\gamma)^{-1}\cdot \bigl(\|\phi_{\beta_k}^\top \beta' - \phi_{\beta'}^\top \beta' \|_{1, \nu_k} + \|\phi_{\beta_k}^\top \beta' - \phi_{\beta'}^\top \beta' \|_{1, \nu_\rE}\bigr) \Bigr] \nonumber\\
	&\quad \le \EE_\init \bigl[(\beta' - \beta_k)^\top \nabla_\beta L(\theta_k, \beta_k)\bigr] + \cO( B_\beta ^{3/2} \cdot m^{-1/4}),
	\end{align*}
	where the last inequality follows from Assumption \ref{asp:linear_nn}, Lemma \ref{lem:linear_nn}, and the fact that $\beta', \beta_k \in S_{B_\beta}$.
	Thus, we complete the proof of Lemma \ref{lem:reward}.
\end{proof}

\subsection{Proof of Lemma \ref{lem:re0}}
\label{sec:pf_lem_re0}
\begin{proof}
	By the update of $\beta_k$ in \eqref{eq:update_beta}, it holds for any $\beta' \in S_{B_\beta}$ that
	\begin{align*}
	\bigl(\beta_k + \eta\cdot \hat \nabla_\beta L(\theta_k, \beta_k) - \beta_{k+1} \bigr)^\top (\beta' - \beta_{k+1})  \le 0,
	\end{align*}
	which further implies that
	\begin{align}
	\label{eq:reward_diff2}
	\eta \cdot ({\beta' - \beta_k})^\top {\nabla_\beta L(\theta_k, \beta_k)} &\le \norm{\beta_k - \beta'}_2^2 - \norm{\beta_{k+1} - \beta'}_2^2 - \norm{\beta_{k+1} - \beta_k}_2^2  \\
	&\quad + \eta \cdot \Bigl( (\beta_{k+1} - \beta_k)^\top \hat\nabla_\beta L(\theta_k, \beta_k) + (\beta_k - \beta')^\top\bigl(\hat \nabla_\beta L(\theta_k, \beta_k)- \nabla_\beta L(\theta_k, \beta_k)\bigr)\Bigr). \nonumber
	\end{align}
	Combining \eqref{eq:reward_diff} and \eqref{eq:reward_diff2}, we have that
	\begin{align*}
	\eta \cdot \bigl(L(\theta_k, \beta_k) - L(\theta_k, \beta')\bigr)\le  \norm{\beta_k - \beta'}_2^2 - \norm{\beta_{k+1} - \beta'}_2^2 - \norm{\beta_{k+1} - \beta_k}_2^2 + \eta \cdot \Delta_k^{\text{(ii)}},
	\end{align*}
	where $\Delta_k^{\text{(ii)}}$ takes the form of
	\begin{align}
	\label{eq:re0_1}
	\Delta_k^{\text{(ii)}} &= \underbrace{{(\beta_{k+1} - \beta_{k})}^\top{\hat \nabla_\beta L(\theta_k, \beta_k)}}_{\displaystyle\text{(ii.a)}} + \underbrace{ ({\beta_k - \beta'})^\top \bigl(\hat \nabla_\beta L(\theta_k, \beta_k)- \nabla_\beta L(\theta_k, \beta_k)\bigr)}_{\displaystyle\text{(ii.b)}} \nonumber \\
	&\quad + \underbrace{(1-\gamma)^{-1}\cdot \bigl(\|\phi_{\beta_k}^\top \beta' - \phi_{\beta'}^\top \beta' \|_{2, \nu_k} + \|\phi_{\beta_k}^\top \beta' - \phi_{\beta'}^\top \beta' \|_{2, \nu_\rE}\bigr)}_{\displaystyle\text{(ii.c)}}
	\end{align}
We now upper bound terms (ii.a), (ii.b), and (ii.c) on the right-hand side of \eqref{eq:re0_1}. Following from Assumption \ref{asp:linear_nn} and Lemma \ref{lem:linear_nn}, we have that
\begin{align}
\label{eq:re0_2}
\EE_\init \bigl[ \|\phi_{\beta_k}^\top \beta' - \phi_{\beta'}^\top \beta' \|_{2, \nu_k} + \|\phi_{\beta_k}^\top \beta' - \phi_{\beta'}^\top \beta' \|_{2, \nu_\rE} \bigr] = \cO(B_\beta^{3/2}\cdot m^{-1/4}),
\end{align}
which upper bounds term (ii.c) of \eqref{eq:re0_1}.
For term (ii.b) of \eqref{eq:re0_1}, we have that
\begin{align}
\label{eq:re0_3}
&\EE\biggl[\Big|({\beta_k - \beta'})^\top \bigl(\hat \nabla_\beta L(\theta_k, \beta_k)- \nabla_\beta L(\theta_k, \beta_k)\bigr)\Big|\biggr] \nonumber \\
& \quad \le \EE\Bigl[ \norm[\big]{ \hat \nabla_\beta L(\theta_k, \beta_k)- \nabla_\beta L(\theta_k, \beta_k) }_2 \cdot \norm{\beta' - \beta_k}_2 \Bigr]  \le 2B_\beta \cdot \EE\bigl[\norm{\xi_k'}_2\bigr]  \le 2B_\beta \cdot (\sigma^2/N)^{1/2},
\end{align}
where we write $\xi_k' = \hat \nabla_\beta L(\theta_k, \beta_k)- \nabla_\beta L(\theta_k, \beta_k)$.
Here the first inequality follows from the Cauchy-Schwartz inequality, the second inequality follows from  the fact that $\beta_k, \beta' \in S_{B_\beta}$, and the last inequality follows from Assumption \ref{asp:more}. 
To upper bound term (ii.a) in \eqref{eq:re0_1}, we have that
\begin{align}
\label{eq:re0_4}
& \EE\Bigl[\bigl| {(\beta_{k+1} - \beta_{k})}^\top{\hat \nabla_\beta L(\theta_k, \beta_k)} \bigr|\Bigr]  \\
&\quad \le  \EE\Bigl[ \norm[\big]{ \hat \nabla_\beta L(\theta_k, \beta_k) }_2 \cdot \norm{ \beta _{k+1} - \beta_k }_2 \Bigr]  \le \eta \cdot \EE \Bigl[ \norm[\big]{ \hat \nabla_\beta L(\theta_k, \beta_k) }_2^2 \Bigr] = 2\eta \cdot \Bigl(\norm[\big]{\nabla_\beta L (\theta_k, \beta_k)}^2_2 + \EE\bigl[\norm{\xi_k'}^2_2\bigr]\Bigr), \nonumber
\end{align}
where the first inequality follows from the Cauchy-Schwartz inequality and the second inequality follows from the update of $\beta$ in \eqref{eq:update_beta}.
 Furthermore, we have
\begin{align}
\label{eq:re0_5}
 \norm[\big]{\nabla_\beta L (\theta_k, \beta_k)}^2_2 &= \norm[\Big]{ \EE_{\nu_k}\bigl[\phi_ {\beta_k}(s, a)\bigr] -\EE_{\nu_E}\bigl[\phi_ {\beta_k}(s, a)\bigr] + \lambda\cdot\nabla_\beta \psi(\beta_k)}_2^2\nonumber\\
 & \le \biggl( \EE_{\nu_k}\Bigl[\norm[\big]{\phi_ {\beta_k}(s, a)}_2\Bigr] + \EE_{\nu_k}\Bigl[\norm[\big]{\phi_ {\beta_k}(s, a)}_2\Bigr] + \lambda \cdot \norm[\big]{\nabla_\beta \psi(\beta_k)}_2\biggr)^2 \nonumber\\
 & \le (2+\lambda\cdot L_\psi)^2,
\end{align}
where the first inequality follows from Jensen's inequality and the second inequality follows from the fact that $ \norm{\phi_W(s, a)}_2 \le 1 $ and the Lipschitz continuity of $\psi(\beta)$ in Assumption \ref{asp:regularizer}. By plugging \eqref{eq:re0_5} into \eqref{eq:re0_4}, we have that
\begin{align}
\label{eq:re0_6}
\EE\Bigl[\bigl| {\hat \nabla_\beta L(\theta_k, \beta_k)}^\top({\beta_{k} - \beta_{k+1}}) \bigr|\Bigr] &\le \eta \cdot \Bigl( (2+\lambda \cdot L_\psi)^2 + \EE\bigl[\norm{\xi_k'}^2_2\bigr]\Bigr) \nonumber \\
&\le \eta \cdot \bigl( (2+\lambda \cdot L_\psi)^2 + \sigma^2 / N\bigr),
\end{align}
where the last inequality follows from Assumption \ref{asp:more}. 
Finally, by plugging \eqref{eq:re0_2}, \eqref{eq:re0_3}, and \eqref{eq:re0_6} into \eqref{eq:re0_1}, we have that
\begin{align*}
\EE_\init\bigl[|\Delta_k^{\text{(ii)}}|\bigr] &= \eta \cdot \bigl( (2+\lambda \cdot L_\psi)^2 + \sigma^2\cdot N^{-1}\bigr) + 2 B_\beta \cdot \sigma\cdot N^{-1/2} + \cO(B_\beta^{3/2}\cdot m^{-1/4}).
\end{align*}
Thus, we complete the proof of Lemma \ref{lem:re0}.
\end{proof}

\section{Proofs of Supporting Lemmas}
In what follows, we present the proofs of the lemmas in \S\ref{sec:pf_aux}.
\subsection{Proof of Lemma \ref{lem:bound_pe1}}
\label{sec:pf_lem_bound_pe1}
\begin{proof}	
	It holds for any policies $\pi, \pi'$ that
	\begin{align}
	\label{eq:pe0_1}
	\bigl\langle D(s), \pi^s - (\pi')^s \bigr\rangle_\cA = 0,
	\end{align}
	where $D(s)$ only depends on the state $s$.
	Thus, we have that
	\begin{align*}
	&\bigl\langle \log(\pi _ {k+1} ^ s / \pi_k^s) - \eta \cdot\hat Q_{\omega_k}(s, \cdot), \pi_\rE^s - \pi_{k}^s \bigr\rangle_\cA \nonumber\\
	&\quad  = \bigl\langle \tau _{k+1}\cdot  \phi _{\theta_{k+1}}(s, \cdot)^\top \theta _{k+1} - \tau _k\cdot  \phi _{\theta_k}(s, \cdot) ^\top \theta_k - \eta \cdot \phi _{\omega_k}(s, \cdot)^\top \omega_k, \pi_\rE^s - \pi_k^s \bigr\rangle_\cA \nonumber\\
	&\quad = \bigl\langle \tau _{k+1}\cdot  \phic _{\theta_{k+1}}(s, \cdot)^\top \theta _{k+1} - \tau _k\cdot  \phic _{\theta_k}(s, \cdot) ^\top \theta_k - \eta \cdot \phic _{\omega_k}(s, \cdot)^\top \omega_k, \pi_\rE^s - \pi_k^s \bigr\rangle_\cA,
	\end{align*}
	where the first inequality follows from the parameterization of $\pi_\theta$  and $\hat Q_\omega$ in \eqref{eq:para_pi} and \eqref{eq:para_q}, respectively, and the second equality follows from the definition of the temperature-adjusted score function $\phic_{\theta}(s, a)$ in \eqref{eq:def_phic} of Proposition \ref{prop:npg_form}.
	Here, with a slight abuse of the notation, we define
	\begin{align}
	\label{eq:def-phiw}
	\phic_{\omega_k}(s, a) = \phi_{\omega_k}(s, a) - \EE_{a'\sim \pi_k(\cdot \given s)}\bigl[\phi_{\omega_k}(s, a')\bigr].
	\end{align}
	Then, following from \eqref{eq:pe0_1} and the update $\tau_{k+1} \cdot \theta_{k+1} = \tau_k\cdot\theta_k - \eta \cdot \delta_k$ in \eqref{eq:update_theta}, we have that
	\begin{align}
	\label{eq:pe1_2}
	&\bigl\langle \log(\pi _ {k+1} ^ s / \pi_k^s) - \eta \cdot \hat Q_{\omega_k}(s, \cdot), \pi_\rE^s - \pi_{k}^s \bigr\rangle_\cA \\
	&\quad  = \bigl\langle \tau _{k+1} \cdot \phic_{\theta_{k+1}}(s, \cdot)^\top \theta _{k+1} - \tau _k \cdot \phic _{\theta_k}(s, \cdot) ^\top \theta_k - \eta \cdot \phic _{\omega_k}(s, \cdot)^\top \omega_k, \pi_\rE^s - \pi_k^s \bigr\rangle_\cA \nonumber \\
	&\quad = \underbrace{\tau_{k+1} \cdot \inp[\big]{\phic_{\theta_{k+1}}(s, \cdot)^\top \theta_{k+1} - \phic_{\theta_k}(s, \cdot)^\top\theta_{k+1} }{\pi_\rE^s - \pi_k^s }_\cA}_{\displaystyle\text{(i)}} - \underbrace{\eta \cdot \inp[\big]{\phic_{\theta_k}(s, \cdot)^\top \delta_k + \phic_{\omega_k}(s, \cdot)^\top \omega_k}{\pi_\rE^s - \pi_k^s }_\cA}_{\displaystyle\text{(ii)}}. \nonumber
	\end{align}
	In what follows, we upper bound terms (i) and (ii) on the right-hand side of \eqref{eq:pe1_2}.
	
	\vskip4pt
	
	\noindent{\bf Upper bound of term (i) in \eqref{eq:pe1_2}.} Following from \eqref{eq:def_phic} of Proposition \ref{prop:npg_form} and \eqref{eq:pe0_1} we have that
	\begin{align}
	\label{eq:pe1_3}
	&\Bigl|\inp[\big]{\phic_{\theta_{k+1}}(s, \cdot)^\top \theta_{k+1} - \phic_{\theta_k}(s, \cdot)^\top\theta_{k+1} }{\pi_\rE^s - \pi_k^s }_\cA \Bigr| \nonumber \\
	&\quad = \Bigl|\inp[\big]{\phi_{\theta_{k+1}}(s, \cdot)^\top \theta_{k+1} - \phi_{\theta_k}(s, \cdot)^\top\theta_{k+1} }{\pi_\rE^s - \pi_k^s }_\cA \Bigr|\nonumber\\
	&\quad \le \norm[\big]{ \phi_{\theta_{k+1}}(s, \cdot)^\top \theta_ {k+1} - \phi_ {\theta_k}(s, \cdot)^\top \theta_{k+1} }_{1, \pi_\rE^s} + \norm[\big]{ \phi_{\theta_{k+1}}(s, \cdot )^\top \theta_ {k+1} - \phi_ {\theta_k}(s, \cdot)^\top \theta_{k+1} }_{1, \pi_k^s} ,
	\end{align}
	where the inequality follows from the triangle inequality.
	Following from Assumption \ref{asp:linear_nn} and Lemma \ref{lem:linear_nn}, we have that
	\begin{align}
	\label{eq:pe1_4}
	&\EE_{\init, d_\rE}\Bigl[\norm[\big]{ \phi_{\theta_{k+1}}(s, \cdot )^\top \theta_ {k+1} - \phi_ {\theta_k}(s, \cdot)^\top \theta_{k+1} }_{1, \pi_\rE^s}\Bigr] = \cO(B_\theta^{3/2}\cdot m^{-1/4}).
	\end{align}
	Furthermore, following from Assumption \ref{asp:visitation}, Lemma \ref{lem:linear_nn}, and the Cauchy-Schwartz inequality, we have that
	\begin{align}
	&\EE_{\init, d_\rE}\Bigl[\norm[\big]{ \phi_{\theta_{k+1}}(s, \cdot)^\top \theta_ {k+1} - \phi_ {\theta_k}(s, \cdot)^\top \theta_{k+1} }_{1, \pi_k^s}\Bigr] \nonumber\\
	&\quad = \EE_{\init, d_k} \biggl[\norm[\big]{ \phi_{\theta_{k+1}}(s, \cdot )^\top \theta_ {k+1} - \phi_ {\theta_k}(s, \cdot)^\top \theta_{k+1} }_{1, \pi_k^s} \cdot \frac{\rd d_\rE}{\rd d_k}\biggr]  \nonumber\\
	&\quad \le \norm[\big]{ \phi_{\theta_{k+1}}(s, a)^\top \theta_ {k+1} - \phi_ {\theta_k}(s, a)^\top \theta_{k+1} }_{2, \nu_k} \cdot \norm[\bigg]{\frac{\rd d_\rE}{\rd d_k}}_{2, d_k}\nonumber\\
	& \quad = \cO(B_\theta^{3/2}\cdot m^{-1/4}). \label{eq:pe1_4.5}
	\end{align}
	Here the expectation $\EE_{\init, d_{k}}$ is taken with respect to the random initialization in \eqref{eq:init} and $s \sim d_{k}$.
	Thus, plugging \eqref{eq:pe1_4} and \eqref{eq:pe1_4.5} into \eqref{eq:pe1_3}, we obtain for term (i) of \eqref{eq:pe1_2} that
	\begin{align}
	\label{pe1_5}
	\EE_{\init, d_\rE} \biggl[\Bigl| \inp[\big]{\phic_{\theta_{k+1}}(s, \cdot)^\top \theta_{k+1} - \phic_{\theta_k}(s, \cdot)^\top\theta_{k+1} }{\pi_\rE^s - \pi_k^s }_\cA \Bigr|\biggr] = \cO( B_\theta^{3/2}\cdot m^{-1/4} ).
	\end{align}
	
\vskip4pt
	
\noindent{\bf Upper bound of term (ii) in \eqref{eq:pe1_2}.} Following from the Cauchy-Schwartz inequality, we have that
	\begin{align}
	\label{eq:pe1_6}
	\EE _{d_\rE} \biggl[ \Bigl| \inp[\big]{\phic_{\theta_k}(s, \cdot)^\top \delta_k + \phic_{\omega_k}(s, \cdot)^\top \omega_k}{\pi_\rE^s }_\cA\Bigr| \biggr] & \le \int_{\cS\times\cA} \big|\phic_{\theta_k}(s, a)^\top \delta_k + \phic_{\omega_k}(s, a)^\top \omega_k \big| \rd \nu_\rE (s, a) \nonumber \\
	&\le \norm[\bigg]{\frac{\rd \nu_\rE}{\rd \nu_k}}_{2, \nu_k} \cdot \norm[\big]{ \phic_{\theta_k}(s, a)^\top \delta_k + \phic_{\omega_k}(s, a)^\top \omega_k }_{2, \nu_k}.
	\end{align}
	 Similarly, we have that
	\begin{align}
	\label{eq:pe1_7}
	\EE _{d_\rE} \biggl[ \Bigl| \inp[\big]{\phic_{\theta_k}(s, \cdot)^\top \delta_k + \phic_{\omega_k}(s, \cdot)^\top \omega_k}{\pi_k^s }_\cA \Bigr| \biggr] &\le \int_{\cS\times\cA} \Bigl| \inp[\big]{\phic_{\theta_k}(s, \cdot)^\top \delta_k + \phic_{\omega_k}(s, \cdot)^\top \omega_k}{\pi_k^s }_\cA\Bigr| \rd \pi^s_k (a) \rd d_\rE(s) \nonumber \\
	& = \int_{\cS\times\cA} \Bigl| \inp[\big]{\phic_{\theta_k}(s, \cdot)^\top \delta_k + \phic_{\omega_k}(s, \cdot)^\top \omega_k}{\pi_k^s }_\cA\Bigr| \cdot \frac{\rd d_\rE}{\rd d_k} (s) \rd \nu_k (s, a) \nonumber \\
	& \le \norm[\bigg]{\frac{\rd d_\rE}{\rd d_k}}_{2, d_k} \cdot \norm[\big]{ \phic_{\theta_k}(s, a)^\top \delta_k + \phic_{\omega_k}(s, a)^\top \omega_k }_{2, \nu_k},
	\end{align}
	where the last inequality follows from the Cauchy-Schwartz inequality.
	Combining \eqref{eq:pe1_6} and \eqref{eq:pe1_7}, we obtain for term (ii) of \eqref{eq:pe1_2} that
	\begin{align}
	\label{eq:pe1_8}
	&\EE _{d_\rE} \biggl[ \Bigl| \inp[\big]{\phic_{\theta_k}(s, \cdot)^\top \delta_k + \phic_{\omega_k}(s, \cdot)^\top \omega_k}{\pi_\rE^s - \pi_k^s }_\cA\Bigr| \biggr]\nonumber\\
	 &\quad \le \Biggl(\norm[\bigg]{\frac{\rd \nu_\rE}{\rd \nu_k}}_{2, \nu_k} + \norm[\bigg]{\frac{\rd d_\rE}{\rd d_k}}_{2, d_k}\Biggr) \cdot \norm[\big]{ \phic_{\theta_k}(s, a)^\top \delta_k + \phic_{\omega_k}(s, a)^\top \omega_k }_{2, \nu_k}\nonumber\\
	&\quad \le C_h \cdot \norm[\big]{ \phic_{\theta_k}(s, a)^\top \delta_k + \phic_{\omega_k}(s, a)^\top \omega_k }_{2, \nu_k},
	\end{align}
	where the last inequality follows from Assumption \ref{asp:measures}.
	To upper bound term (ii) of \eqref{eq:pe1_2}, it suffices to upper bound the right-hand side of \eqref{eq:pe1_8}. For notational simplicity, we write $\phic_{\theta_k} = \phic_{\theta_k}(s,a)$, $\phic_{\omega_k} = \phic_{\omega_k}(s,a)$, and $\phi_{\omega_{k}} = \phi_{\omega_{k}}(s, a)$.  By the triangle inequality, we have that
	\begin{align}
	\label{eq:pe1_9}
	&\norm{ \delta_k^\top\phic_{\theta_k} + \omega_k^\top\phic_{\omega_k} }_{2, \nu_k} = \Bigl(\EE _{ \nu_k} \bigl[ ( \delta_k^\top\phic_{\theta_k} + \omega_k^\top\phic_{\omega_k} ) \cdot ( \delta_k^\top\phic_{\theta_k} + \omega_k^\top\phic_{\omega_k} )\bigr]\Bigr)^{1/2} \nonumber \\
	&\quad \le \underbrace{\Bigl| (\delta_k - \omega_k)^\top \EE_ {\nu_k} \bigl[\phic _{\theta_k} (\delta_k^\top \phic _{\theta_k} + \omega_k^\top \phic _{\omega_k}  ) \bigl] \Bigr|^{1/2}}_{\displaystyle\text{(ii.a)}} + \underbrace{\Bigl|\EE _{\nu_k}\bigl[\omega_k^\top( \phic_ {\theta_k} - \phic_{\omega_k})\cdot (\delta_k^\top \phic _{\theta_k} + \omega_k^\top \phic_{\omega_k})\bigr]\Bigr|^{1/2}}_{\displaystyle\text{(ii.b)}}.
	\end{align}
	We now upper bound the two terms (ii.a) and (ii.b) on the right-hand side of \eqref{eq:pe1_9}.
	For term (ii.a) of \eqref{eq:pe1_9}, following from \eqref{eq:fisher_form} of Proposition~\ref{prop:npg_form}, we have that
	\begin{align}
	\label{eq:pe1_10}
	\cI(\theta_k) &= \tau_k^2 \cdot \EE_{\nu_k} [\phic_{\theta_k} \phic_{\theta_k}^\top ].
	\end{align}
	Recall that the expectation $\EE_k$ is taken with respect to the $k$-th batch. Following from the definition of $\hat \nabla_\theta L(\theta_k, \beta_k)$ in \eqref{eq:est_pg}, we have that
	\begin{align}
	\label{eq:pe1-10}
	\EE_{k}\bigl[\hat \nabla_\theta L(\theta_k, \beta_k)\bigr] &= -\tau_k \cdot \EE _{\nu_k} [  \omega_k^\top \phi_{\omega_k} \cdot \phic_{\theta_k} ]\nonumber \\
	&=-\tau_k \cdot \EE _{\nu_k} [  \omega_k^\top \phic_{\omega_k} \cdot \phic_{\theta_k} ] - \tau_k \cdot w_k^\top \EE_{a' \sim \pi_k^s}\bigl[\phi_{\omega_k}(s, a')\bigr] \cdot  \EE_{ \nu_k}[  \phic_{\theta_k} ]\nonumber\\
	 &= -\tau_k \cdot \EE _{\nu_k} [  \omega_k^\top \phic_{\omega_k} \cdot \phic_{\theta_k} ],
	\end{align}
	where the first equality follows from the fact that $\hat Q_{\omega_k}(s, a) = \omega_k^\top\phi_{\omega_k}(s, a)$, while the second and third equalities follow from the definition of $\phic_{\omega_k}(s, a)$ in \eqref{eq:def-phiw}.
	Following from \eqref{eq:pe1_10} and \eqref{eq:pe1-10}, we have that
	\begin{align}
	\label{eq:pe1_11}
	\Bigl| (\delta_k - \omega_k)^\top \EE_ {\nu_k} \bigl[\phic _{\theta_k} (\delta_k^\top \phic _{\theta_k} + \omega_k^\top \phic _{\omega_k}  ) \bigl] \Bigr| &= \tau_k^{-2} \cdot \biggl| (\delta_k - \omega_k)^\top \Bigl(\cI(\theta_k)\delta_k - \tau_k \cdot \EE_k\bigl[\hat \nabla_\theta L(\theta, \beta)\bigr] \Bigr) \biggr| \nonumber \\
	& \le 2B_\theta \cdot \tau_k^{-2} \cdot \norm[\Big]{ \cI(\theta_k) \delta_k - \tau_k \cdot \EE_{k}\bigl[\hat \nabla_\theta L(\theta, \beta)\bigr] }_2.
	\end{align}
	Here the last inequality follows from the Cauchy-Schwartz inequality and the fact that $\norm{\omega_k - \delta_k}_2 \le 2 B_\theta$ as $\omega_{k}, \delta_{k} \in S_{B_{\theta}}$. For notational simplicity, we define the following error terms,
	\begin{align}
	\label{eq:pe1_err1}
	\xi_k^{(1)} &= \hat \cI(\theta_k) \delta_k - \cI(\theta_k) \delta_k, \\
	\label{eq:pe1_err2}
	\xi_k^{(2)} &= \hat \nabla_\theta L(\theta_k, \beta_k) - \EE_k\bigl[\hat \nabla_\theta L(\theta_k, \beta_k)\bigr] .
	\end{align}
	Then, we have for term (ii.a) in \eqref{eq:pe1_9} that
	\begin{align}
	\label{eq:pe1_12}
	&\EE_{\init}\biggl[ \Bigl| (\delta_k - \omega_k)^\top \EE_ { \nu_k} \bigl[\phic _{\theta_k} (\delta_k^\top \phic _{\theta_k} + \omega_k^\top \phic _{\omega_k}  ) \bigl] \Bigr|^{1/2} \biggr]  \\
	&\quad \le (2B_\theta)^{1/2}\cdot \tau_k^{-1} \cdot \EE_{\init} \biggl[ \norm[\Big]{ \cI(\theta_k)\delta_k - \tau_k \cdot \EE_{k}\bigl[\hat \nabla_\theta L(\theta, \beta)\bigr] }_2^{1/2} \biggr] \nonumber \\
	&\quad \le (2B_\theta)^{1/2}\cdot \tau_k^{-1} \cdot \EE_\init \biggl[ \Bigl(\norm[\big]{ \hat \cI(\theta_k)\delta_k - \tau_k \cdot \hat \nabla_\theta L(\theta, \beta)}_2 + \norm{\xi_k^{(1)}}_2 + \tau_k\cdot \norm{\xi_k^{(2)}}_2 \Bigr)^{1/2} \biggr] \nonumber \\
	&\quad \le (2B_\theta)^{1/2}\cdot \tau_k^{-1} \cdot \biggl( \EE_{\init}\Bigl[ \norm[\big]{ \hat \cI(\theta_k)\delta_k - \tau_k \cdot \hat \nabla_\theta L(\theta, \beta)}_2 \Bigr] + \EE_{\init}\bigl[\norm{\xi_k^{(1)}}_2 + \tau_k\cdot \norm{\xi_k^{(2)}}_2 \bigr] \biggr)^{1/2}, \nonumber
	\end{align}
	where the first inequality follows from \eqref{eq:pe1_11}, the second inequality follows from the triangle inequality, and the last inequality follows from Jensen's inequality. Similarly to \eqref{eq:pe1_err1}, we define the following error term,
	\begin{align}
	\label{eq:pe1_err3}
	\xi_k^{(3)} = \hat \cI(\theta_k)  \omega_k - \cI(\theta_k)  \omega_k.
	\end{align}
	We now upper bound the right-hand side of \eqref{eq:pe1_12}. Recall the definition of $\delta_k$ in \eqref{eq:def_delta}. We have that
	\begin{align}
	\label{eq:pe1_13}
	\norm[\big]{ \hat \cI(\theta_k)  \delta_k - \tau_k \cdot \hat \nabla_\theta L(\theta_k, \beta_k)}_2 &\le \norm[\big]{ \hat \cI(\theta_k) \omega_k - \tau_k \cdot \hat \nabla_\theta L(\theta_k, \beta_k)}_2  \\
	 &\le \norm[\Big]{  \cI(\theta_k) \omega_k - \tau_k \cdot \EE_{k}\bigl[\hat \nabla_\theta L(\theta_k, \beta_k)\bigr]}_2 + \norm{\xi_k^{(1)}}_2 + \tau_k\cdot\norm{\xi_k^{(2)}}_2. \nonumber
	\end{align}
	Following from \eqref{eq:pe1_10}, \eqref{eq:pe1-10}, and Jensen's inequality, we have that
	\begin{align*}
	\norm[\Big]{  \cI(\theta_k) \omega_k - \tau_k \cdot \EE_k \bigl[\hat \nabla_\theta L(\theta_k, \beta_k)\bigr]}_2 &= \tau_k^2 \cdot \norm[\Big]{ \EE_{ \nu_k}\bigl[\phic_{\theta_k} \cdot \omega_k^\top (\phic_{\theta_k} - \phic _{\omega_k})\bigr] }_2 \nonumber \\
	& \le \tau_k^2 \cdot \EE _{\nu_k }\Bigl[\norm{\phic_{\theta_k}}_2 \cdot \bigl|\omega_k^\top (\phic_{\theta_k} - \phic _{\omega_k})\bigr| \Bigr] \nonumber \\
	& \le 2\tau_k^2\cdot  \norm[\big]{\omega_k^\top (\phic_{\theta_k} - \phic _{\omega_k})}_{1, \nu_k},
	\end{align*}
	where the last inequality follows from the fact that $\norm{\phic_\theta}_2\le 2$ for any $(s, a)\in \cS\times \cA$. Following from Assumption \ref{asp:linear_nn} and Lemma \ref{lem:linear_nn}, we have that
	\begin{align}
	\label{eq:pe1_15}
	\EE_{\init}\biggl[ \norm[\Big]{  \cI(\theta_k) \omega_k - \tau_k \cdot \EE_k\bigl[\hat \nabla_\theta L(\theta_k, \beta_k)\bigr]}_2 \biggr] &\le \EE_{\init}\Big[2\tau_k^2\cdot  \norm[\big]{\omega_k^\top (\phic_{\theta_k} - \phic _{\omega_k})}_{1, \nu_k}\Big] \nonumber \\
	&= \cO(\tau_k^2 \cdot B_\theta^{3/2}\cdot m^{-1/4}).
	\end{align}
	Plugging  \eqref{eq:pe1_13} and \eqref{eq:pe1_15} into \eqref{eq:pe1_12}, we have that
	\begin{align}
	\label{eq:pe1_16}
	&\EE_{\init}\biggl[ \Bigl| (\delta_k - \omega_k)^\top \EE_ { \nu_k} \bigl[\phic _{\theta_k} (\delta_k^\top \phic _{\theta_k} + \omega_k^\top \phic _{\omega_k}  ) \bigl] \Bigr|^{1/2} \biggr] \nonumber \\
	&\quad = (2B_\theta)^{1/2}\cdot \tau_k^{-1}\cdot \Bigl(\cO(2 \tau_k^2\cdot B_\theta^{3/2}\cdot m^{-1/4}) + \EE_{\init}\bigl[\norm{\xi_k^{(1)}}_2 + 2\tau_k \cdot \norm{\xi_k^{(2)}}_2 +\norm{\xi_k^{(3)}}_2  \bigr] \Bigr)^{1/2} \nonumber\\
	&\quad = \cO(\tau_k\cdot B_\theta^{5/4}\cdot m^{-1/4}) + (2B_\theta)^{1/2}\cdot \tau_k^{-1}\cdot \Bigl(\EE_{\init}\bigl[\norm{\xi_k^{(1)}}_2 + 2\tau_k \cdot \norm{\xi_k^{(2)}}_2 +\norm{\xi_k^{(3)}}_2 \bigr]\Bigr)^{1/2} \nonumber\\
	&\quad \le \cO(\tau_k\cdot B_\theta^{5/4}\cdot m^{-1/4}) + 2\sqrt{2} B_\theta^{1/2}\cdot (\sigma^2 / N)^{1/4},
	\end{align}
	where the last inequality follows from Assumption \ref{asp:more}.
	We now upper bound term (ii.a) of \eqref{eq:pe1_9}. We have that
	\begin{align}
	\label{eq:pe1_17}
	&\EE_{\init}\biggl[\Bigl|\EE _{\nu_k}\bigl[\omega_k^\top( \phic_ {\theta_k} - \phic_{\omega_k})\cdot (\delta_k^\top \phic _{\theta_k} + \omega_k^\top \phic_{\omega_k})\bigr]\Bigr|^{1/2}\biggr] \nonumber \\
	&\quad\le \EE_{\init, \nu_k}\Bigl[\bigl|\omega_k^\top( \phic_ {\theta_k} - \phic_{\omega_k})\cdot (\delta_k^\top \phic _{\theta_k} +\omega_k^\top \phic_{\omega_k})\bigr|\Bigr]^{1/2} \nonumber \\
	&\quad \le \EE_{\init}\Bigl[ \norm[\big]{\omega^\top_k (\phic _{\theta_k} - \phic_{\omega_k})}_{2, \nu_k}\Bigr]^{1/2}\cdot \EE_{\init}\bigl[\norm{ \delta_k^\top \phic_{\theta_k} + \omega_k^\top \phic_{\omega_k} }_{2, \nu_k} \bigr]^{1/2},
	\end{align}
	where the expectation $\EE_{\init, \nu_k}$ is taken with respect to the random initialization in \eqref{eq:init} and $(s, a) \sim \nu_k$, the first inequality follows from Jensen's inequality, and the second inequality follows from the Cauchy-Schwartz inequality. Following from Assumption \ref{asp:linear_nn} and Lemma \ref{lem:linear_nn}, we have that
	\begin{align}
	\label{eq:pe1_18}
	\EE_{\init}\Bigl[ \norm[\big]{\omega^\top_k (\phic _{\theta_k} - \phic_{\omega_k})}_{2, \nu_k}\Bigr] = \cO(B_\theta^{3/2}\cdot m^{-1/4}).
	\end{align}
	To upper bound the right-hand side of \eqref{eq:pe1_17}, it remains to upper bound the term $\EE_{\init}[\norm{ \delta_k^\top \phic_{\theta_k} + \omega_k^\top \phic_{\omega_k} }_{2, \nu_k} ] $. We have that
	\begin{align}
	\label{eq:pe1_20}
	 \EE_{\init}\bigl[\norm{ \delta_k^\top \phic_{\theta_k} + \omega_k^\top \phic_{\omega_k} }_{2, \nu_k} \bigr]  \le \EE_\init \bigl[ \norm{\delta_k}_2 \cdot \norm{\phic_{\theta_k}}_2 \bigr] + \EE_\init \bigl[ \norm{\omega_k}_2 \cdot \norm{\phic_{\omega_k}}_2 \bigr]  = \cO(B_\theta),
	\end{align}
	where the inequality follows from the Cauchy-Schwartz inequality and the equality follows from the facts that $\norm{\phic_{\theta_k}}_2 \le 2$, $\norm{\phic_{\omega_k}}_2 \le 2$, and $\delta_k, \omega_k \in S_{B_\theta}$.
	Plugging  \eqref{eq:pe1_18} and \eqref{eq:pe1_20} into \eqref{eq:pe1_17}, we have that
	\begin{align}
	\label{eq:pe1_21}
	\EE_{\init}\biggl[\Bigl|\EE _{\nu_k}\bigl[\omega_k^\top( \phic_ {\theta_k} - \phic_{\omega_k}) \cdot (\delta_k^\top \phic _{\theta_k} + \omega_k^\top \phic_{\omega_k})\bigr]\Bigr|^{1/2}\biggr] = \cO(B_\theta^{5/4} \cdot m^{-1/8}),
	\end{align}
	which upper bounds term (ii.b) of \eqref{eq:pe1_9}.
	Plugging  \eqref{eq:pe1_16} and \eqref{eq:pe1_21} into \eqref{eq:pe1_9}, following from \eqref{eq:pe1_8},  we have that
	\begin{align}
	\label{eq:pe1_22}
	&\EE _{\init,d_\rE} \biggl[ \Bigl| \inp[\big]{\phic_{\theta_k}(s, \cdot)^\top \delta_k - \phic_{\omega_k}(s, \cdot)^\top \omega_k}{\pi_\rE^s - \pi_k^s }_\cA\Bigr| \biggr] \nonumber \\
	&\quad = \eta \cdot C_h \cdot \bigl(\cO(B_\theta^{5/4}\cdot m^{-1/8} )+ 2\sqrt{2} B_\theta^{1/2}\cdot (\sigma^2/N)^{1/4}\bigr),
	\end{align}
	which upper bounds term (ii) of \eqref{eq:pe1_2}.

    \vskip4pt
	
	Finally, plugging \eqref{pe1_5} and \eqref{eq:pe1_22} into \eqref{eq:pe1_2}, we have that
	\begin{align*}
	&\EE_{\init,  d_\rE}\biggl[\Bigl| \inp[\big]{\log (\pi_ {k+1}^s / \pi _{k}^s) - \eta \cdot \hat Q_{\omega_k}(s, \cdot) }{ \pi_\rE^s - \pi_k^s }_\cA \Bigr|\biggr] \nonumber \\
	&\quad = \eta\cdot C_h \cdot 2\sqrt{2} B_\theta^{1/2}\cdot (\sigma^2/N)^{1/4} + \cO(\tau_{k+1} \cdot B_\theta^{3/2} \cdot m^{-1/4} + \eta \cdot B_\theta^{5/4} \cdot m^{-1/8} ),
	\end{align*}
	where $\xi_k^{(1)}$, $\xi_k^{(2)}$, and $\xi^{(3)}_k$ are defined in \eqref{eq:pe1_err1}, \eqref{eq:pe1_err2}, and \eqref{eq:pe1_err3}, respectively, and $C_h$ is defined in Assumption \ref{asp:visitation}. Thus, we complete the proof of Lemma \ref{lem:bound_pe1}.
\end{proof}	

\subsection{Proof of Lemma \ref{lem:bound_pe2}}
\label{sec:pf_lem_bound_pe2}
\begin{proof}
	For notational simplicity, for any $(s, a) \in \cS\times \cA$, we denote by $\Delta_{Q, k}(s, a) = \hat  Q_{\omega_k}(s, a) - Q^{\pi_k}_{r_k}(s, a)$ the error of estimating $ Q^{\pi_k}_{r_k}(s, a) $ by $ \hat  Q_{\omega_k}(s, a) $. Then, we have that
	\begin{align*}
	&\EE_{d_\rE} \biggl[ \Bigl| \inp[\big]{\Delta_{Q, k}(s, \cdot) }{\pi_\rE^s - \pi_k^s}_\cA \Bigr| \biggr] \nonumber\\
	 &\quad  \le \int_{\cS\times\cA} \bigl|\Delta_{Q, k}(s, a) \bigr| \rd \pi_\rE^s (a) \rd d_\rE(s) + \int_{\cS\times\cA} \bigl|\Delta_{Q, k}(s, a) \bigr| \rd \pi_k^s (a) \rd d_\rE(s) \nonumber \\
	&\quad  = \int_{\cS\times\cA} \bigl|\Delta_{Q, k}(s, a) \bigr| \cdot \frac{\rd \nu_\rE}{\rd \rho_k}(s, a) \rd \rho_k(s, a) + \int_{\cS\times\cA} \bigl|\Delta_{Q, k}(s, a) \bigr| \cdot \frac{\rd d_\rE}{\rd  \varrho_k}(s) \rd \rho_k(s, a) \nonumber \\
	&\quad \le C_h \cdot \norm{\Delta_{Q, k}}_{2, \rho_k},
	\end{align*}
	where the last inequality follows from the Cauchy-Schwartz inequality and Assumption \ref{asp:visitation}. Thus, we complete the proof of Lemma \ref{lem:bound_pe2}.
\end{proof}

\subsection{Proof of Lemma \ref{lem:bound_pe3}}
\label{sec:pf_lem_bound_pe3}
\begin{proof}
	Following from \eqref{eq:pe0_1} and the parameterization of $\pi_\theta$ in \eqref{eq:para_pi}, we have that
	\begin{align}
	\label{eq:pe3_1}
	&\inp[\big]{\log (\pi_{k+1}^s / \pi_k^s)}{\pi_k^s - \pi_{k+1}^s}_\cA\\ &\quad= \inp[\big]{\tau_{k+1} \cdot \theta_{k+1}^\top \phi_ {\theta_ {k+1}}(s, \cdot) -\tau_k \cdot \theta_k^\top \phi _{\theta_k}(s, \cdot) }{\pi_k^s - \pi_ {k+1}^s}_\cA \nonumber \\
	& \quad= \inp[\big]{(\tau _{k+1}\cdot\theta_{k+1} - \tau_k\cdot \theta_k)^\top \phi_{\theta_k}(s, \cdot) }{\pi_k^s - \pi_ {k+1}^s }_\cA + \tau_{k+1}\cdot \inp[\Big]{\theta_ {k+1}^\top \bigl(\phi_ {\theta_{k+1}}(s, \cdot) - \phi_ {\theta_k}(s, \cdot)\bigr)}{\pi_k^s - \pi_ {k+1}^s}_\cA. \nonumber
	\end{align}
	We now upper bound the two terms on the right-hand side of \eqref{eq:pe3_1}. 
	For the first term on the right-hand side of \eqref{eq:pe3_1}, recall that we define $\delta_k$ in \eqref{eq:def_delta}. Thus, we have that
	\begin{align}
	\label{eq:pe3_2}
	\bigl|(\tau _{k+1}\cdot\theta_{k+1} - \tau_k\cdot\theta_k)^\top \phi_{\theta_k}(s, a) \bigr| = \eta \cdot \bigl| \delta_k^\top \phi_{\theta_k}(s, a) \bigr|.
	\end{align}
Following from \eqref{eq:pe3_2} and H\"older's inequality, we have for any $s \in \cS$ that
\begin{align*}
&\Bigl|\inp[\big]{(\tau _{k+1}\cdot\theta_{k+1} - \tau_k\cdot\theta_k)^\top \phi_{\theta_k}(s, \cdot) }{\pi_k^s - \pi_ {k+1}^s }_\cA \Bigr|\nonumber \\
&\quad \le \norm[\big]{ \delta_k^\top \phi_{\theta_k}(s, \cdot) }_{\infty} \cdot \norm{\pi_k^s - \pi_ {k+1}^s}_{1}.
\end{align*}
Then, following from Pinsker's inequality, we have that
\begin{align}
\label{eq:pe3-5}
&\Bigl|\inp[\big]{(\tau _{k+1}\cdot\theta_{k+1} - \tau_k\cdot\theta_k)^\top \phi_{\theta_k}(s, \cdot) }{\pi_k^s - \pi_ {k+1}^s }_\cA \Bigr| - \kl(\pi_ {k+1}^s \,\|\, \pi_k^s) \nonumber \\
&\quad \le \eta \cdot \norm[\big]{ \delta_k^\top \phi_{\theta_k}(s, \cdot) }_{\infty} \cdot \norm{\pi_k^s - \pi_ {k+1}^s}_{1} - 1/2\cdot \norm{\pi_k^s - \pi_ {k+1}^s}_{1}^2 \nonumber \\
&\quad \le 1/2\cdot \eta^2 \cdot \norm[\big]{ \delta_k^\top \phi_{\theta_k}(s, \cdot) }_{\infty}^2.
\end{align}
By the update of $\theta_k$ in \eqref{eq:update_theta} and the definition of $\delta_k$ in \eqref{eq:def_delta}, we have that $\theta_k, \delta_k \in S_{B_\theta}$. Thus, by Lemma \ref{lem:bound_nn}, we have that
\begin{align}
\label{eq:pe3-51}
 \EE_{\init}\Bigl[\norm[\big]{\delta_k^\top \phi_{\theta_k}(s,\cdot)}_{\infty}^2\Bigr] \le 2M_0 + 18B_\theta^2.
\end{align}
Plugging \eqref{eq:pe3-51} into \eqref{eq:pe3-5}, we have that
\begin{align}
\label{eq:pe3_5}
\Bigl|\inp[\big]{(\tau _{k+1}\cdot\theta_{k+1} - \tau_k\cdot\theta_k)^\top \phi_{\theta_k}(s, \cdot) }{\pi_k^s - \pi_ {k+1}^s }_\cA \Bigr| - \kl(\pi_ {k+1}^s \,\|\, \pi_k^s) \le \eta^2 \cdot (M_0^2 + 9 B_\theta^2).
\end{align}
For the second term on the right-hand side of \eqref{eq:pe3_1}, following from Assumption \ref{asp:linear_nn} and Lemma \ref{lem:linear_nn}, we have
\begin{align}
\label{eq:pe3_6}
&\EE_{\init, d_\rE}\Biggl[\biggl| \inp[\Big]{\theta_ {k+1}^\top \bigl(\phi_ {\theta_{k+1}}(s, \cdot) - \phi_ {\theta_k}(s, \cdot)\bigr)}{\pi_k^s - \pi_ {k+1}^s}_\cA \biggr|\Biggr]\nonumber \\
&\quad \le \EE_{\init, d_\rE}\biggl[ \norm[\Big]{\theta_ {k+1}^\top \bigl(\phi_ {\theta_{k+1}}(s, \cdot) - \phi_ {\theta_k}(s, \cdot)\bigr)}_{1, \pi_k^s } \biggr] + \EE_{\init, d_\rE}\biggl[ \norm[\Big]{\theta_ {k+1}^\top \bigl(\phi_ {\theta_{k+1}}(s, \cdot) - \phi_ {\theta_k}(s, \cdot) \bigr)}_{1, \pi_{k+1}^s } \biggr] \nonumber \\
&\quad = \cO(B_\theta^{3/2} \cdot m^{-1/4}).
\end{align}
Finally, plugging  \eqref{eq:pe3_5} and \eqref{eq:pe3_6} into \eqref{eq:pe3_1}, we have that
\begin{align*}
&\EE_{\init, d_\rE}\biggl[\Bigl|\inp[\big]{\log (\pi_{k+1}^s / \pi_k^s)}{\pi_k^s - \pi_{k+1}^s}_\cA \Bigr| -\kl(\pi_ {k+1}^s \,\|\, \pi_k^s) \biggr]  \nonumber\\ 
&\quad = \eta^2 \cdot (M_0^2 +9B_\theta^2) + \cO(\tau_{k+1}\cdot B_\theta^{3/2}\cdot m^{-1/4}),
\end{align*}
which completes the proof of Lemma \ref{lem:bound_pe3}.
\end{proof}
\end{document}